\documentclass[acmtog,authorversion]{acmart}
\acmSubmissionID{357}

\usepackage{booktabs} % For formal tables
\usepackage[printonlyused,smaller,withpage]{acronym}
\usepackage{amsfonts}
\usepackage{amsmath}
\usepackage{amssymb}
\usepackage{gensymb}
\usepackage{bm}
\usepackage{graphicx}
\usepackage{subcaption}

\definecolor{revisionColor}{rgb}{.8,.3,.3}
\newcommand{\revised}[1]{#1}

\usepackage[ruled]{algorithm2e} % For algorithms

\SetAlFnt{\small}
\SetAlCapFnt{\small}
\SetAlCapNameFnt{\small}
\SetAlCapHSkip{0pt}

\acmJournal{TOG}
\acmYear{2020}
\acmVolume{39}
\acmNumber{4}
\acmArticle{92}
\acmMonth{7}
\acmDOI{10.1145/3386569.3392437}

\setcopyright{acmlicensed}
\received{January 2020}
\received[revised]{April 2020}
\received[accepted]{April 2020}

\newcommand{\eg}{\textit{e.g.}}
\newcommand{\ie}{\textit{i.e.}}
\newcommand{\ii}{\mathsf{i}\mkern1mu}

\begin{document}
% Title portion

\title{CNNs on Surfaces using Rotation-Equivariant Features}

% DO NOT ENTER AUTHOR INFORMATION FOR ANONYMOUS TECHNICAL PAPER SUBMISSIONS TO SIGGRAPH 2019!
\author{Ruben Wiersma}
\orcid{0000-0001-7900-7253}
\email{r.t.wiersma@tudelft.nl}
\author{Elmar Eisemann}
\orcid{0000-0003-4153-065X}
\email{e.eisemann@tudelft.nl}
\author{Klaus Hildebrandt}
\email{k.a.hildebrandt@tudelft.nl}
\affiliation{%
 \institution{Delft University of Technology}}

\renewcommand\shortauthors{R. Wiersma, E. Eisemann, \& K. Hildebrandt}

\begin{abstract}
\revised{
This paper is concerned with a fundamental problem in geometric deep learning that arises in the construction of convolutional neural networks on surfaces. Due to curvature, the transport of filter kernels on surfaces results in a rotational ambiguity, which prevents a uniform alignment of these kernels on the surface. We propose a network architecture for surfaces that consists of vector-valued, rotation-equivariant features. The equivariance property makes it possible to locally align features, which were computed in arbitrary coordinate systems, when aggregating features in a convolution layer. The resulting network is agnostic to the choices of coordinate systems for the tangent spaces on the surface. We implement our approach for triangle meshes. Based on circular harmonic functions, we introduce convolution filters for meshes that are rotation-equivariant at the discrete level. We evaluate the resulting networks on shape correspondence and shape classifications tasks and compare their performance to other approaches.}
\end{abstract}

\begin{CCSXML}
<ccs2012>
<concept>
<concept_id>10010147.10010257.10010293.10010294</concept_id>
<concept_desc>Computing methodologies~Neural networks</concept_desc>
<concept_significance>500</concept_significance>
</concept>
<concept>
<concept_id>10010147.10010371.10010396.10010402</concept_id>
<concept_desc>Computing methodologies~Shape analysis</concept_desc>
<concept_significance>500</concept_significance>
</concept>
</ccs2012>
\end{CCSXML}

\ccsdesc[500]{Computing methodologies~Neural networks}
\ccsdesc[500]{Computing methodologies~Shape analysis}


\keywords{Geometric Deep Learning, CNNs on Surfaces, Surface Networks, Rotation-Equivariance, Circular Harmonic Filters, Shape Classification, Shape Segmentation, Shape Correspondence}

\maketitle
%!TEX root = main.tex
\section{Introduction}
The success of Deep Learning approaches based on convolutional neural networks (CNNs) in computer vision and image processing has motivated the development of analogous approaches for the analysis, processing, and synthesis of surfaces. Along these lines, approaches have been proposed for problems such as shape recognition \cite{Su2015}, shape matching \cite{boscaini2016learning}, shape segmentation \cite{maron2017convolutional}, shape completion \cite{litany2017deformable}, curvature estimation \cite{Guerrero2018}, and 3D-face synthesis \cite{ranjan2018generating}. %and the prediction of rigid body simulations \cite{Rempe2019}  

In contrast to images, which are described by regular grids in a Euclidean domain, surfaces are curved manifolds and the grids on these surfaces are irregular. In order to still use regular grids, one can work with multiple projections of the surface on planes \cite{Su2015} or with volumetric grids \cite{wu20153d}. 

An alternative to learning on regular grids is generalized deep learning, often referred to as geometric deep learning \cite{bronstein2017geometric}, which targets irregularly sampled manifolds and graphs. A central element of such geometric CNNs is a generalized convolution operator. For CNNs on images, the convolution layers are built from convolution kernels, which are transported across the image. As a result, the parameters that define one kernel describe the convolution across the whole image, which significantly reduces the number of parameters that need to be learned. This is a motivation for exploring constructions of generalized convolution operators on surfaces based on convolution kernels. 
%In this paper, we focus on a fundamental problem that arises in the construction of such generalized convolution operators. 

\begin{figure}
    \centering
    \includegraphics[width=\columnwidth]{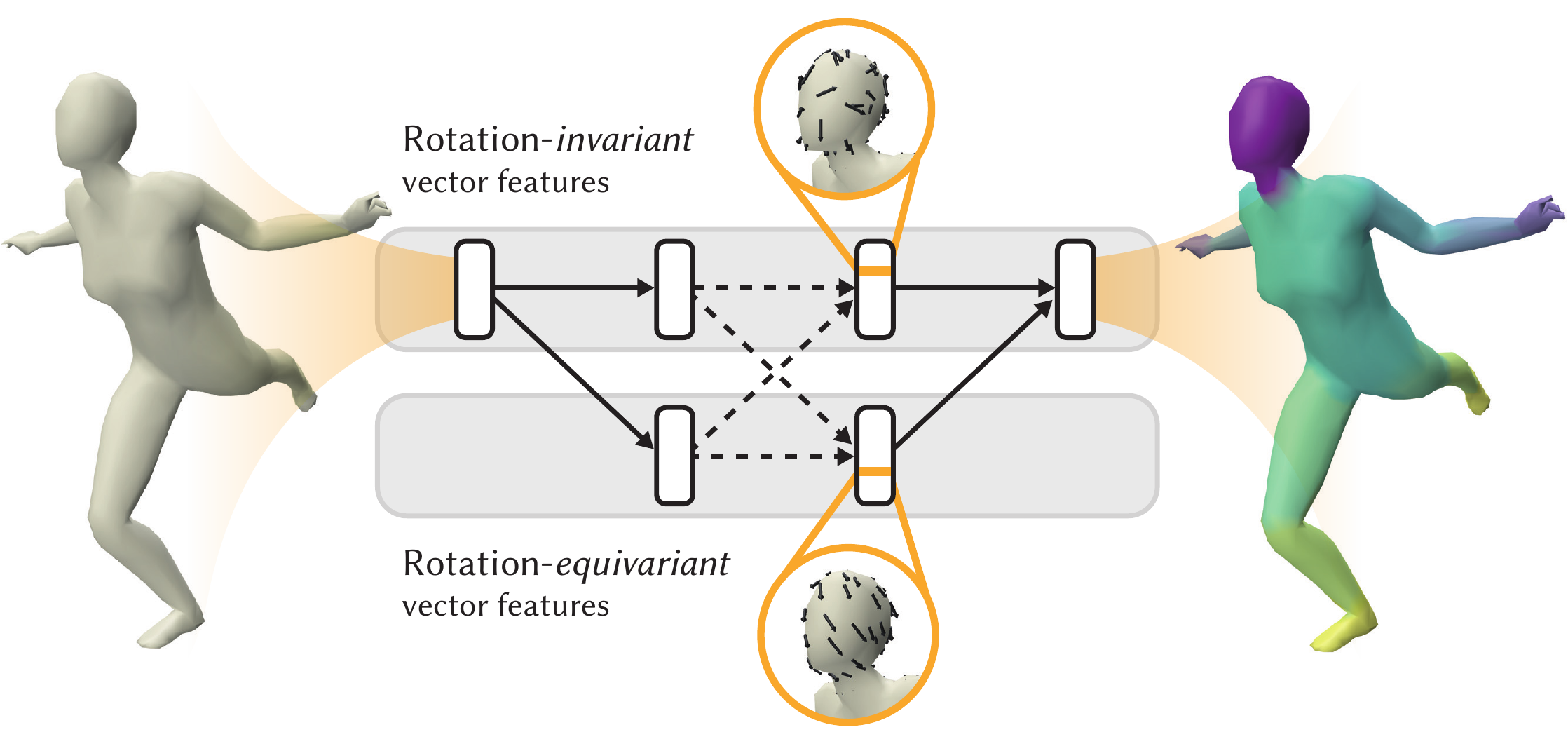}
    \Description{A mesh of a human shape is input to a neural network with two streams, a rotation-invariant stream and a rotation-equivariant stream. An example of the features inside each stream is shown as a vector field on the mesh. The output of the network is a colored mesh, a different color for each body part.}
    \caption{We propose CNNs on surfaces that operate on vectors and separate rotation-equivariant and rotation-invariant features.}
    \label{fig:intro}
\end{figure}
\revised{
To apply a convolution kernel defined on $\mathbb{R}^2$ to a function at a point on a surface, the Riemannian exponential map is used to locally lift the function from the surface to a function defined on the tangent plane at the point. By identifying the tangent plane with $\mathbb{R}^2$, the convolution of the kernel and the lifted function can be computed. In this way, the convolution kernel can be applied everywhere on the surface. However, a problem arises, since there is a rotational degree of freedom when $\mathbb{R}^2$ is identified with a tangent plane. Moreover, the transport of filters on a surface depends on the chosen path. If a filter is transported along two different ways from one point of a surface to another, the transported filters are rotated against each other. This \textit{rotation ambiguity} problem is fundamental and caused by the curvature of the surface. }

The rotation ambiguity problem can be addressed by specifying a coordinate system at each point of the surface, \eg~according to the principal curvature directions \cite{boscaini2016learning,monti2017geometric,Pan18} or the direction of maximum activation \cite{masci2015geodesic,Sun18}. As a consequence, however, the coordinate systems in the local neighborhoods are arranged in different patterns for each point. For a network this means that, when features are aggregated to form the next layer of the network, the features are not only dependent on the sequence of convolution kernels that are applied, but also on the arrangement of coordinate systems in the local neighborhoods. Loosely speaking, the information contained in the features in the neighborhood of a point can be arbitrarily rotated against the coordinate system at the point. One can think of this as in a cubist painting, where the elements that make up a structure, for example the eyes, nose, and mouth on a face, are rotated against each other.

\revised{
Multi-Directional Geodesic CNNs (MDGCNN) \cite{poulenard2018multi} provide an alternative approach to the rotation ambiguity problem. The idea is to sample the rotational degree of freedom regularly, to compute the convolutions for each of the sample directions, and to feed the results for all directions into the network. The multi-directional features are pooled at the final layer. 
The disadvantage of this approach is that each filter must not only be evaluated once at each point, but also for rotated versions of the filter and the results need to be stored. To build an operable network, the filters are only evaluated in some sample directions and the results are linearly interpolated to get results in intermediate directions, introducing inaccuracies in each consecutive layer. }

\revised{
We introduce a novel network architecture that does not suffer from the rotation ambiguity problem. The features in this network are rotation-equivariant and organized in streams of different equivariance classes (\autoref{fig:intro}).
These streams interact with each other in the network and are finally merged into a rotation-invariant output. 
The resulting network is independent of the choice of coordinate systems in the tangent spaces, which means that it does not suffer from the rotation ambiguity problem.
To realize this network, we work with vector-valued, instead of scalar-valued, features and introduce rotation-equivariant convolution and pooling operators on meshes. 
The convolution operators use the Riemannian exponential map and parallel transport to convolve vector-valued kernels on $\mathbb{R}^2$ with tangent vector fields on meshes. 

As kernels on $\mathbb{R}^2$ we use the circular harmonics, which are known to be rotation-equivariant. 
We prove that the resulting discrete convolution operators on meshes are equivariant with respect to rotations of the coordinate system in the tangent spaces. 
Due to the rotation-equivariance property, it suffices to compute the convolution at each point with respect to an arbitrary reference coordinate system in the tangent plane. Then, if the result with respect to any other coordinate system is required, one only needs to transform the result of the convolution in the reference coordinate system to the other coordinate system. The rotation-equivariance property is still valid if several of these filters are applied consecutively, \eg~in the deeper layers of a network. 
Rotation-equivariance enables the vector-valued convolution operator to always align the features in a local neighborhood of a point to the coordinate system at the point. 
Our network architecture builds upon Harmonic Nets~\cite{worrall2017harmonic}, which are used for rotation-equivariant learning on images. Therefore, we call the networks \emph{Harmonic Surface Networks} (HSN). }

We implement the Harmonic Surface Networks for triangle meshes. Based on  circular harmonic filters, we derive discrete convolution filters that are equivariant with respect to basis transformations in the tangent spaces associated to the vertices of the mesh. The parameters of the filters separate radial and angular direction, which leads to a low number of parameters for each kernel, when compared to other filter kernels for surface meshes. We experimentally analyze the properties and performance of the HSNs and compare the performance to other geometric CNNs for surface meshes. 

In summary, our main contributions are:
\begin{itemize}
\item  We introduce Harmonic Surface Networks, which combine a vector-valued convolution operation on surfaces and rotation-equivariant filter kernels to provide a solution to the rotation ambiguity problem of geometric deep learning on surfaces. 
\item  Based on circular harmonics, we derive convolution filters for meshes that have the desired rotational-equivariance properties at the discrete level, and, additionally, allow for separating learning of parameters in radial and angular direction.
\item  We analyze the Harmonic Surface Networks for surface meshes and compare the performance to alternative approaches.  
\end{itemize}
%!TEX root = main.tex
\section{Related work}
In this section, we provide a brief overview of work on geometric deep learning closely related to our approach. For a comprehensive survey on the topic, we refer to \cite{bronstein2017geometric}.

\paragraph*{Charting-based methods} 
Closest to our work are so-called \emph{charting-based methods}. For convolution, these methods use local parametrizations to apply filter kernels to features defined on the surface. A seminal approach in this direction are the Geodesic CNNs (GCNN, \cite{masci2015geodesic}). They consider a parametric model of convolution filters in $\mathbb{R}^2$ using Gaussian kernels placed on a regular polar grid. For convolution, the filters are mapped to the surface using the Riemannian exponential map. 
Other intrinsic methods use anisotropic \revised{heat} kernels \cite{boscaini2016learning}, learn the shape parameters of the kernels \cite{monti2017geometric}, learn not only the parameters of the kernels but also pseudo-coordinates of the local parametrizations \cite{verma2018feastnet}, or use B-Splines \cite{Fey2018} or Zernike polynomials \cite{Sun18} instead of Gaussians. 

A challenging problem for these constructions is the lack of canonical coordinate systems on a surface, inducing the rotation ambiguity problem, discussed in the introduction. To address the ambiguity, maximum pooling over the responses to a sample of different rotations can be used \cite{masci2015geodesic,Sun18} or the coordinate systems can be aligned to the \revised{(smoothed)} principal curvature directions \cite{boscaini2016learning,monti2017geometric,Pan18}. 
Both approaches have their downsides. Angular pooling discards directional information and the alignment to principal curvature directions can be difficult as the principal curvature directions are not defined at umbilic points and can be unstable in regions around umbilic points. A sphere, for example, consists only of umbilic points. 
\revised{A solution to
this problem are the Multi-Directional Geodesic CNNs (MDGCNN) \cite{poulenard2018multi}. At every point, the filters are evaluated with respect to multiple choices of coordinate systems. This can be formalized by describing the features by so-called directional functions instead of regular functions. Parallel transport is used to locally align the directional functions for convolution.}
%TODO compare more to us / sampling and interpolation of directions
Our approach provides a different solution to the rotation ambiguity problem that does not require computing the filters in multiple directions and storing the results. 
Once the filter in one direction is computed, the rotation-equivariance property of our filters allows us to obtain the results for the other directions by applying a rotation. We compare our approach to MDGCNN in Section~\ref{sect.experiments}. 
Parallel Transport Vector Convolution~\cite{schonsheck2018parallel} is a convolution operation for vector fields on manifolds that was used to learn from vector valued data on surfaces. Similar to this approach, we also use parallel transport to define convolution for vector fields on surface.
A concept for the construction of Gauge-equivariant convolution filters on manifolds is introduced in \cite{Cohen2019} along with a concrete realization of a Gauge CNN for the icosahedron. In recent work, an adaption of Gauge CNNs to surface meshes was introduced \cite{haan2020gauge}.

Instead of local charts, one can also parametrize an area globally and then learn on the parameter domain \cite{maron2017convolutional,DBLP:journals/corr/abs-1812-10705}. 
An advantage of this approach is that standard CNNs can be applied to the parameter domain. A disadvantage is that global parametrizations lead to larger metric distortion than local parametrizations. In addition, typically different global parametrization methods are needed for surfaces of different genus. 

\paragraph*{Spectral methods and graph CNNs}
An alternative to charting-based methods are spectral methods. For images, CNNs can operate in the Fourier domain. Spectral Networks \cite{bruna2013spectral} generalize CNNs to graphs using the spectrum of a graph Laplacian. To reduce the computational complexity, ChebNet \cite{defferrard2016convolutional} and Graph Convolutional Networks (GCN) \cite{kipf2016semi} use local filters based on the graph Laplacian.
Recently, various approaches for defining CNNs on graphs have been introduced, we refer to \cite{Zhou2018GraphNN,wu2019comprehensive} for recent surveys. 
This line of work diverges from our approach since we specialize to discrete structures that describe two-dimensional manifolds. Furthermore, we aim at 
analyzing the surface underlying a mesh, which means our method aims to be agnostic of the connectivity of the mesh, \ie~the graph underlying a mesh.
The concept of local filtering has also been applied to surfaces using the Laplace--Beltrami and Dirac operator \cite{kostrikov2017surface}.
MeshCNN \cite{hanocka2019meshcnn} generalizes graph CNNs to mesh structures by defining convolution and pooling operations for triangle meshes. 

\paragraph*{Point clouds}
PointNet \cite{qi2017pointnet} is an approach for creating CNNs for unordered sets of points. First, a neighborhood is constructed around each point with a radius ball. To retrieve a response for point $i$, a function is applied to each neighbor of $i$ and the maximum activation across $i$'s neighbors is stored as the new response for $i$. PointNet++ \cite{qi2017pointnet++} is an extension of PointNet with hierarchical functionality. Because PointNet applies the same function to each neighbor, it is effectively rotation-invariant.
DGCNN \cite{wang2018dynamic} extends PointNet++ by dynamically constructing neighborhood graphs within the network. This allows the network to construct and learn from semantic neighborhoods, instead of mere spatial neighborhoods.
PointCNN \cite{Li2018} learns a $\chi$-transformation from the point cloud and applies the convolutions to the transformed points.
\revised{TextureNet \cite{Huang19} uses 4-rotational symmetric fields to define the convolution domains and applies operators that are invariant to 4-fold symmetries.}
While we evaluate our approach on meshes, our method can be used to process point clouds sampled from a surface.

\revised{\paragraph{Symmetric spaces}}
Specialized approaches for symmetric surfaces such as the sphere \cite{Cohen2018,Kondor2018} have been proposed. On the one hand, the approaches profit from the highly symmetric structure of the surfaces, while on the other hand they are limited to data defined on these surfaces.

\paragraph{Rotation-equivariance}
Our approach builds on Harmonic Networks \cite{worrall2017harmonic}. This work is part of a larger effort to create group-equivariant networks. 
Different approaches such as steerable filters \cite{freeman1991design, liu20122d, cohen2016steerable, DBLP:conf/nips/WeilerC19}, hard-baking transformations in CNNs \cite{cohen2016group, marcos2016learning, fasel2006rotation, laptev2016ti}, and learning generalized transformations \cite{hinton2011transforming}
have been explored. 
Most relevant to our approach are steerable filters, since we use features that can be transformed, or steered, with parallel transport. The core idea of steerable filters is described by \cite{freeman1991design} and applied to learning by \cite{liu20122d}. The key ingredient for these steerable filters is to constrain them to the family of circular harmonics. \cite{worrall2017harmonic} added a rotation offset and multi-stream architecture to develop Harmonic Networks. The filters in Harmonic Networks are designed in the continuous domain and mapped to a discrete setting using interpolation.
Harmonic Networks was built on by \cite{thomas2018tensor} for Tensor Field Networks. Tensor Field Networks achieve rotation- and translation-equivariance for 3D point clouds by moving from the family of circular harmonics to that of spherical harmonics. In doing so, they lose the phase offset parameter, as it is not commutative in SO(3). Spherical harmonics were also used for rigid motion-invariant processing of point clouds~\cite{Poulenard2O19}.
%!TEX root = main.tex
\section{Background}
% First do method
\paragraph*{Harmonic Networks} 
Harmonic Nets (H-Nets) \cite{worrall2017harmonic} are rotation-equivariant networks that
can be used to solve computer vision tasks, such as image classification or segmentation, in such
a way that a rotation of the input does not affect the output of the network. H-Nets restrict their filters to circular harmonics, resulting in the following filter definition:
\begin{equation}
W_{m}(r,\theta,R,\beta)=R(r)e^{\ii(m\theta+\beta)}%
,\label{eq:circularharmonic}%
\end{equation}
where $r$ and $\theta$ are polar coordinates, $R:\mathbb{R}_{+}\rightarrow
\mathbb{R}$ is the \textit{radial profile}, $\beta\in\lbrack0,2\pi)$ is a
\textit{phase offset}, and $m\in\mathbb{Z}$ is the rotation order. 
The cross-correlation of $W_{m}$ with a complex function $x$ at a point $p$ is given by the integral
\begin{equation}
\lbrack W_{m}\star x](p)=\int_{0}^{\epsilon}\int_{0}^{2\pi}R(r)e^{\ii(m\theta
+\beta)}x(r,\theta)\,r\,d\theta\,dr.
\label{eq:circularharmonicIntegral}%
\end{equation}
This filter is rotation-equivariant
with respect to rotations of the domain of the input to the filter:
\begin{equation}
\lbrack W_{m}\star x^{\phi}](p)=e^{\ii m\phi}[W_{m}%
\star x^{0}](p),\label{eq:rotation}%
\end{equation}
where $x^{0}(r,\theta)$ is a complex function and $x^{\phi
}(r,\theta)=x(r,\theta-\phi)$ is the same function defined on a
rotated domain. The rotation order~$m$ determines how the output of the
filters changes when the domain of the input is rotated. For $m=0$, the result
does not change and the filter is rotation invariant. For $m\geq1$ the result
is rotated by an angle that is $m$ times the original angle, which we refer to as $m$-equivariance. 
% the domain is rotated

\begin{figure}[h]
\centering
\includegraphics[width=\columnwidth]{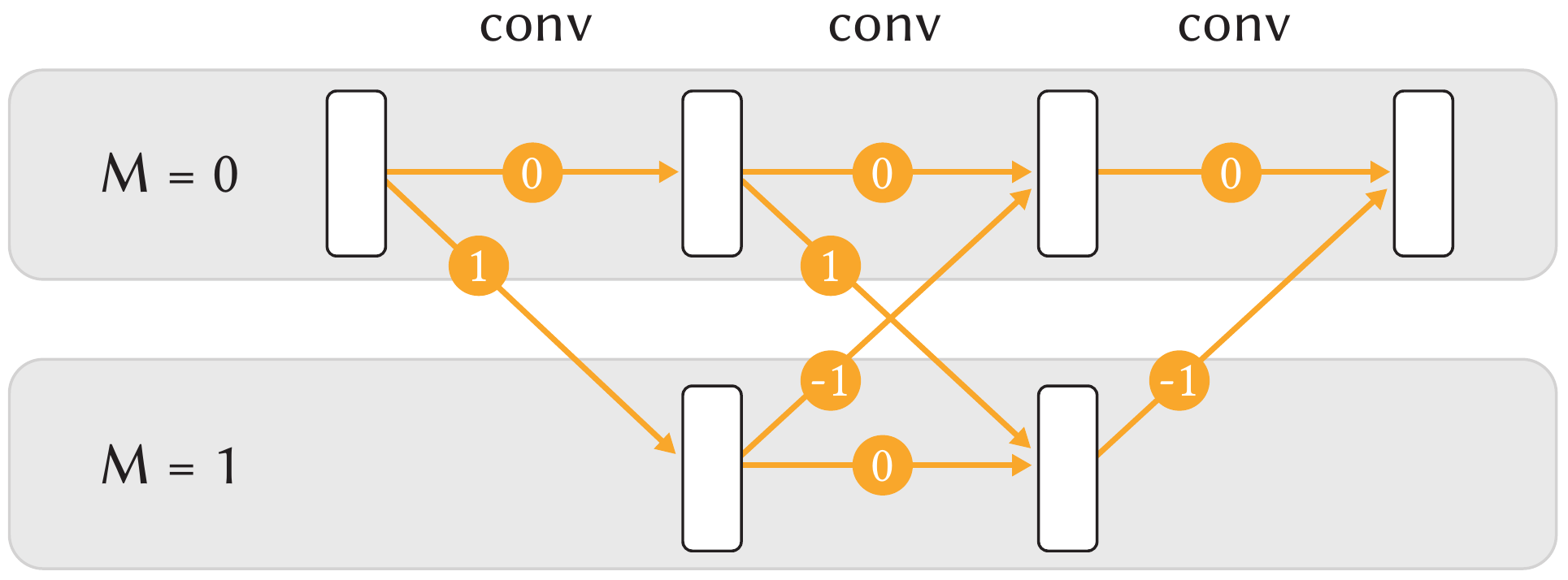}
\Description{A schematic overview of Harmonic Networks: two streams with sequences of feature vectors that are connected with convolutions. The convolutions are labeled by their rotation order. A convolution from the M=0 stream with order 1 crosses to the M=1 stream.}
\caption{Harmonic Networks separate the result of different rotation order convolutions into
streams of $M$-equivariance.}%
\label{fig:streams}
\end{figure}

An important property of these filters is that, if filters of orders $m_{1}$ and $m_{2}$ are chained, rotation-equivariance of order $m_{1}+m_{2}$ is obtained:
\begin{align}
\lbrack W_{m_{1}}\star[W_{m_{2}}\star x^{\phi}] &  =e^{\ii
m_{1}\phi}e^{\ii m_{2}\phi} \lbrack W_{m_{1}}\star[W_{m_{2}}\star x^{0}]\\
&  =e^{\ii(m_{1}+m_{2})\phi}\lbrack W_{m_{1}}\star[W_{m_{2}}\star x^{0}].
\end{align}
This is integral to the rotation-equivariance property of the network as a whole. The network architecture of H-Nets is structured in separate streams per
rotation order as illustrated in \autoref{fig:streams}. For each feature 
map inside a stream
of order $M$, we require that the sum of rotation orders $m_{i}$ along any
path reaching that feature map equals $M$. In the last
layer of an H-Net, the streams are fused to have the same rotation order. The rotation order of the output stream determines the class of equivariance of the network and is chosen to match the demands
of the task at hand. For rotation invariant tasks, the last
layer has order $m=0$. Still, in the hidden layers of the network,
rotation-equivariant streams capture and process information that is not
rotation-invariant, which yields improved performance compared to networks
built only from rotation-invariant filters.

\begin{figure}[b]
    \centering
    \includegraphics[width=0.6\columnwidth]{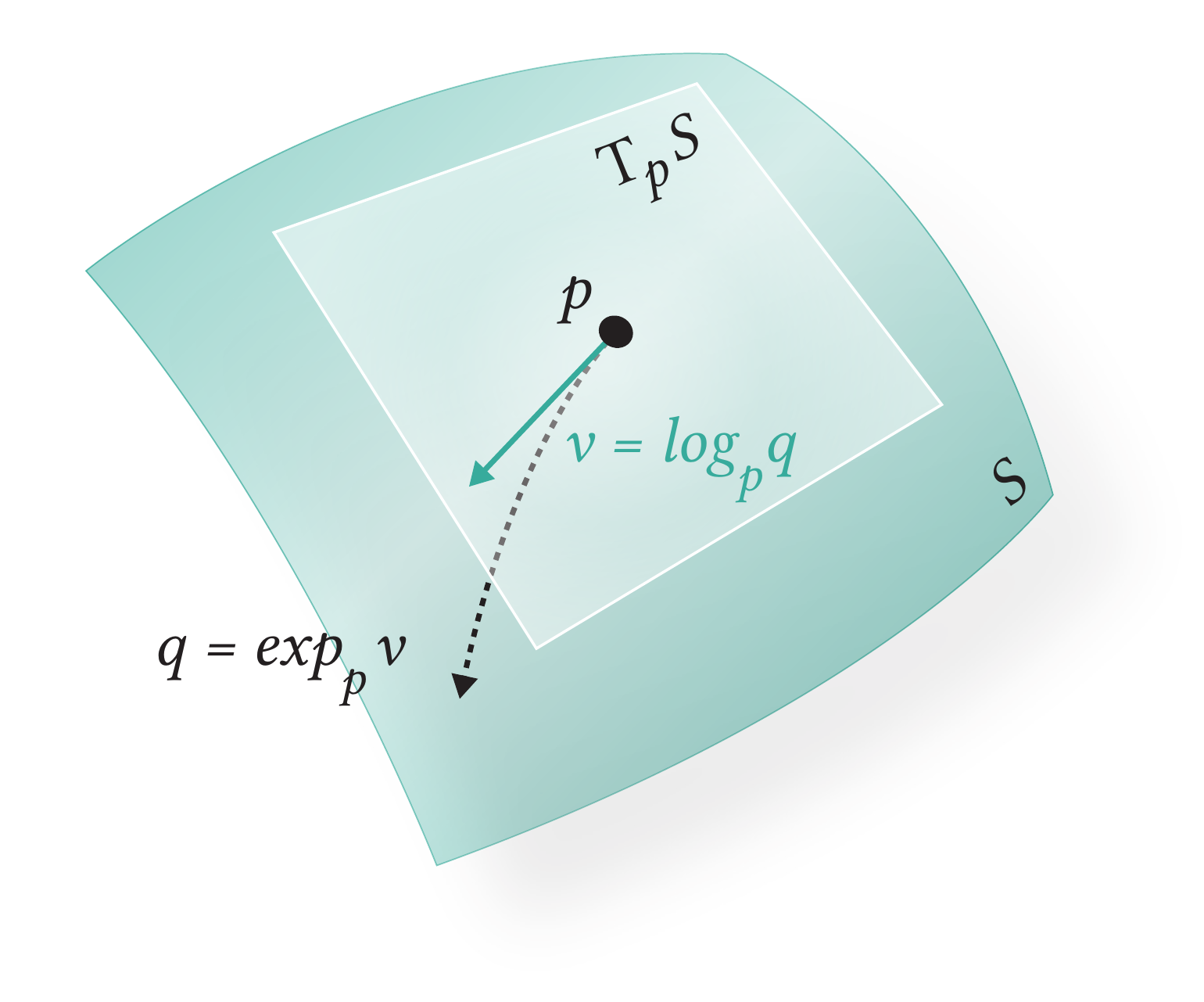}
    \Description{A simplified surface. There are two points on the surface: p and q. A vector v in the tangent plane points from p in the direction of q. The exponential map maps vector v to point q.}
    \caption{The Riemannian exponential- and logarithmic map.}
    \label{fig:expmap}
\end{figure}

\paragraph*{Riemannian exponential map}

Let $v$ be a vector in the tangent plane $T_p\mathcal{S}$ at a point $p$ 
of a surface $\mathcal{S}$. Then, there is exactly one geodesic curve starting at 
$p$ in direction $v$. If we follow this geodesic curve until we have 
covered a length that equals the norm of $v$, we end up at a point $q$ 
on the surface. The Riemannian exponential map associates each vector $v$ in $
T_p\mathcal{S}$ to the corresponding point $q$ on the surface. This map is 
suitable as a local parametrization of the surface, because it is a 
local diffeomorphism, \ie, bijective and smooth with smooth inverse. 
Furthermore, the mapping is an isometry in radial direction away from $p
$. The construction of the Riemannian exponential map and its 
inverse, the logarithmic map, is illustrated in \autoref{fig:expmap}.
An example of the Riemannian exponential map is the azimuthal 
projection of the sphere, which is used in cartography and is included 
in the emblem of the United Nations.

In graphics, the Riemannian 
exponential map is used, for example, for texture decalling~\cite{
schmidt2006interactive} and shape modeling~\cite{schmidt2010meshmixer}. 
In geometric deep learning~\cite{bronstein2017geometric}, it is used to map 
filter kernels defined on the tangent plane to filters defined on the 
surface and to lift functions defined on the surface to functions 
defined on the tangent plane. 
Recent approaches for computing the Riemannian exponential map on 
surface meshes are based on heat diffusion~\cite{
Sharp:2019:VHM,herholzefficient}. These methods reduce the computation 
of the exponential map to solving a few sparse linear systems \revised{and can be accurately computed globally}. Alternatives are approaches based on Dijkstra's algorithm~\cite{
schmidt2006interactive,melvaer2012geodesic}.

\paragraph*{Parallel transport of vectors}

In the Euclidean plane one can describe vector fields by specifying $x$ and $y$ coordinates relative to a global coordinate system. There is no such coordinate system on curved surfaces. To be able to compare vectors at neighboring points $p$ and $q$ of a surface, we use the parallel transport along the shortest geodesic curve $c$ connecting the points. If $v$ is a tangent vector at $p$ which has the angle $\alpha$ with the tangent vector of $c$ at $p$, then the vector transported to $q$ is the tangent vector that has an angle of $\alpha$ with the tangent of $c$ at $q$ and has the same length as $v$. 

To describe tangent spaces on meshes and to compute the parallel transport, we use the Vector Heat Method \cite{
Sharp:2019:VHM}, which we also use to calculate the Riemannian exponential map.

%!TEX root = main.tex

\section{Method}
In this section, we describe the building blocks of Harmonic Surface Networks. We start by
introducing notation and then discuss convolutions, \revised{linearities,} non-linearities, and
pooling. Finally, we discuss why the networks provide a solution to the
rotation ambiguity problem by analyzing how the choices of coordinate systems in
the tangent spaces affect the convolution operations and HSNs.

\paragraph*{Notation}

The features in an HSN are associated to the vertices (or just a subset of the
vertices) of a triangle mesh. To each vertex, we assign a coordinate system
in the tangent plane at the vertex and use complex numbers to represent
tangent vectors with respect to the coordinate system. We denote the
feature vector of rotation order $M$ in network layer $l$ at vertex $i$ by
$\mathbf{x}_{i,M}^{l}$. To simplify the notation, we leave out the indices $l$
and $M$ when they are not relevant.
For HSN, these feature vectors are
complex-valued. Every operation applied to these vectors is performed
element-wise. For example, when we multiply a feature vector with a complex number,
each component of the vector is multiplied by this number. We use parallel
transport along the shortest geodesic from vertex $j$ to vertex $i$ to
transport the feature vectors. The transport depends on the geometry of the
mesh, the chosen coordinate systems at the vertices, and the rotation order
of the feature. It amounts to a rotation of the features. In practice, we 
store the rotation as a complex number with the angle of rotation 
$\phi_{ji}$ as its argument and apply rotations to vectors via complex 
multiplication:
\begin{equation}
P_{j\rightarrow i}(\mathbf{x}_{j,M})=e^{\ii(M\phi_{ji})}\mathbf{x}%
_{j,M}.\label{eq.transport}%
\end{equation}
For every vertex, we construct a parametrization of a local neighborhood of the
vertex over the tangent plane using the Riemannian logarithmic map. We
represent the map by storing polar coordinates $r_{ij}$ and $\theta_{ij}$ with
respect to the coordinate system in the tangent plane for every vertex $j$ in
the neighborhood of vertex $i$.

\paragraph*{Convolution}

The convolution layers of HSNs combine the rotation-equivariance of
circular harmonics with the transport of features along surfaces. 
We use compactly supported filter kernels to limit the number of
neighboring features that are used in a convolutional layer. Let
$\mathcal{N}_{i}$ denote the set of vertices that contribute to the convolution
at $i$. In the case of continuous features on a Euclidean space,
the correlation with $W_{m}$ is given by an integral, see
(\ref{eq:circularharmonicIntegral}). For discrete meshes, we approximate the
integral with a sum and use parallel transport and the Riemannian logarithmic
map to evaluate the filter:% 
\begin{equation}
\mathbf{x}_{i,M+m}^{(l+1)}=\sum_{j\in\mathcal{N}_{i}}w_{j}\left(
R(r_{ij})e^{\ii(m\theta_{ij}+\beta)}P_{j\rightarrow i}(\mathbf{x}_{j,M}%
^{(l)})\right)  ,\label{eq:approxintegr}%
\end{equation}
where $r_{ij}$ and $\theta_{ij}$ are the polar coordinates of point $j$ in $T_{i}\mathcal{S}$, and $P_{j\rightarrow
i}$ is the transport of features defined in~(\ref{eq.transport}). 
The integration weights $w_{j}$ are 
\begin{equation}
w_{j}=\frac{1}{3}\sum_{jkl}A_{jkl} \text{ \qquad\ }
\label{eq:weighting}
\end{equation}
where $A_{jkl}$ is the area of triangle jkl and the sum runs over all triangles of the mesh containing vertex $j$. The weights can be derived from numerical integration with piecewise linear finite elements, see \cite{Wardetzky2007}.
%are derived from a Delaunay tesselation $T$ of the
%points after they are mapped to the tangent plane with the Riemannian
%logarithmic map%

The radial profile $R$ and the phase offset $\beta$ are the learned parameters of the filter. To represent the radial profile, we learn values at equally spaced locations in the interval $\lbrack0,\epsilon\rbrack$, where $\epsilon$ provides a bound on the support of the filter. To get a continuous radial profile, the learned values are linearly interpolated as illustrated in \autoref{fig:radialprofilehsn}. At a point with radial coordinate $r_{ij}$, the radial profile is
\begin{equation}
    R(r_{ij}) = \sum_{q}^Q \mu_{q}(r_{ij}) \rho_{q},
\label{eq:rprofile}
\end{equation}
where $\mu_{q}(r_{ij})$ is a linear interpolation weight and $\rho_{q}$ a learned weight matrix. 
We set $\rho_Q = 0$, which bounds the support of the kernel. 
%Note that the radial profile could be defined in other ways, like a mixture of Gaussians, similar to \cite{monti2017geometric}, or a spline similar to \cite{fey2018splinecnn}.

\begin{figure}[t]
    \centering
    \includegraphics[width=0.6\columnwidth]{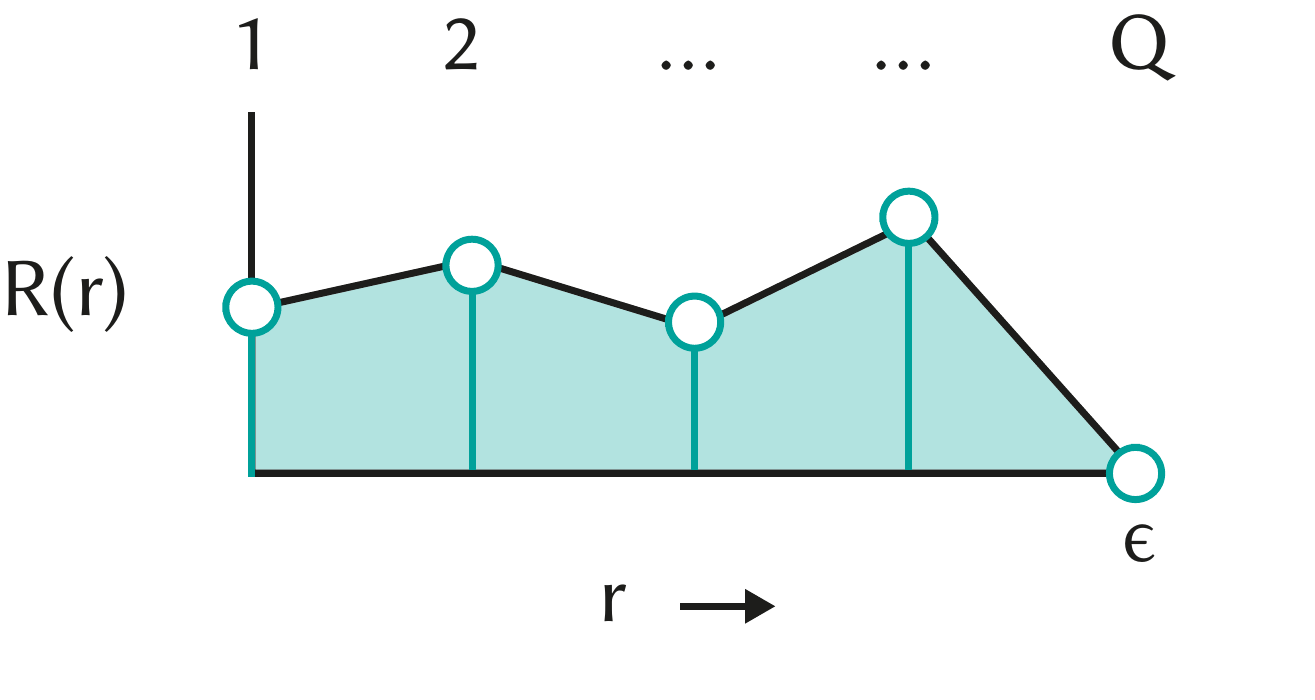}
    \Description{A graph of a function that is linearly interpolated from Q control points.}
    \caption{We parametrize the radial profile by learning the value at $Q$ equally spaced rings and linearly interpolating for values in between.}
    \label{fig:radialprofilehsn}
\end{figure}

To speed up training, we precompute three components: the
integration weight, the interpolation weight to each radial profile point, and
the rotation by the angle $\theta_{ij}$:%
\begin{equation}
\text{precomp}_{ij}=\left(  w_{j}\mu_{q}(r_{ij})e^{\ii m\theta_{ij}}\right)  .
\end{equation}
We precompute the polar coordinates and integration weights with the Vector
Heat method \cite{Sharp:2019:VHM}. The precomputation is stored in a [Q x 2]
matrix for each $(i,j)$ pair.

\revised{
\paragraph*{Linearities}
Linearities are applied to complex features by applying the same linearity to both the real and imaginary component, resulting in a linear combination of the complex features.
}

\paragraph*{Non-Linearities}

We follow Harmonic Networks \cite{worrall2017harmonic} for complex
non-linearities: non-linearities are applied to the radial component of the complex
features, in order to maintain the rotation-equivariance property. In our experiments, we used a complex version of ReLU%
\begin{equation}
\mathbb{C}\text{-ReLU}_{b}(\mathbf{X}e^{\ii\theta})=\text{ReLU}%
(\mathbf{X}+b)e^{\ii\theta},
\end{equation}
where $b$ is a learned bias added to the radial component. If the non-linearity were to be applied without bias, it would be an identity operation, as the radius is always positive.

\paragraph*{Pooling and unpooling}

\label{section:fpstech} Pooling layers downsample the input by aggregating
regions to representative points. We need to define an aggregation operation
suitable for surfaces and a way to choose representative points and create
regions. We use farthest point sampling and cluster all non-sampled points to
sampled points using geodesic nearest neighbors.

The aggregation step in pooling is performed with parallel transport,
since the complex features of points within a pooling region do not exist in
the same coordinate system. Thus, we define the aggregation step for
representative point $i$ with points $j$ in its pooling cluster $\mathcal{C}%
_{i}$ as follows:
\begin{equation}
\mathbf{x}_{i,M}=\square_{j\in\mathcal{C}_{i}}P_{j\rightarrow i}%
(\mathbf{x}_{j,M}),
\end{equation}
where $\square$ is any aggregation operation, such as sum, max, or mean. In our implementation, we use mean pooling. Pooling happens \textit{within} each rotation order stream, hence the rotation
order identifier $M$ for both $\mathbf{x}_{i}$ and $\mathbf{x}_{j}$.

The sampling of points per pooling level and construction of corresponding kernel supports are performed as a precomputation step, so we can compute the
logarithmic map and parallel transport for each pooling level in precomputation. 

\revised{Unpooling layers upsample the input by propagating the features from the points sampled in the pooling layer to their nearest neighbors. Parallel transport is again applied to align coordinate systems.}

\paragraph*{Rotation orders}

To maintain rotation-equivariance throughout the network, the output of the filters is separated in streams of rotation
orders. The output of a filter applied to $\mathbf{x}_{j,M}$ with rotation order $m$ should
end up in rotation order stream $M^{\prime}=M+m$. The output from two
convolutions resulting in the same stream is summed. For example, a
convolution on $\mathbf{x}_{j,1}$ with $m=-1$ and a convolution on
$\mathbf{x}_{j,0}$ with $m=0$ both end up in the stream
$M'=0$ and are summed. \revised{A visual overview can be found in \autoref{fig:resnet}, on the right.} We only apply parallel transport to inputs from the
$M>0$ rotation-order streams, as the values in the $M=0$ stream are
rotation-invariant.

\begin{figure}[t]
    \centering
    \includegraphics[width=\columnwidth]{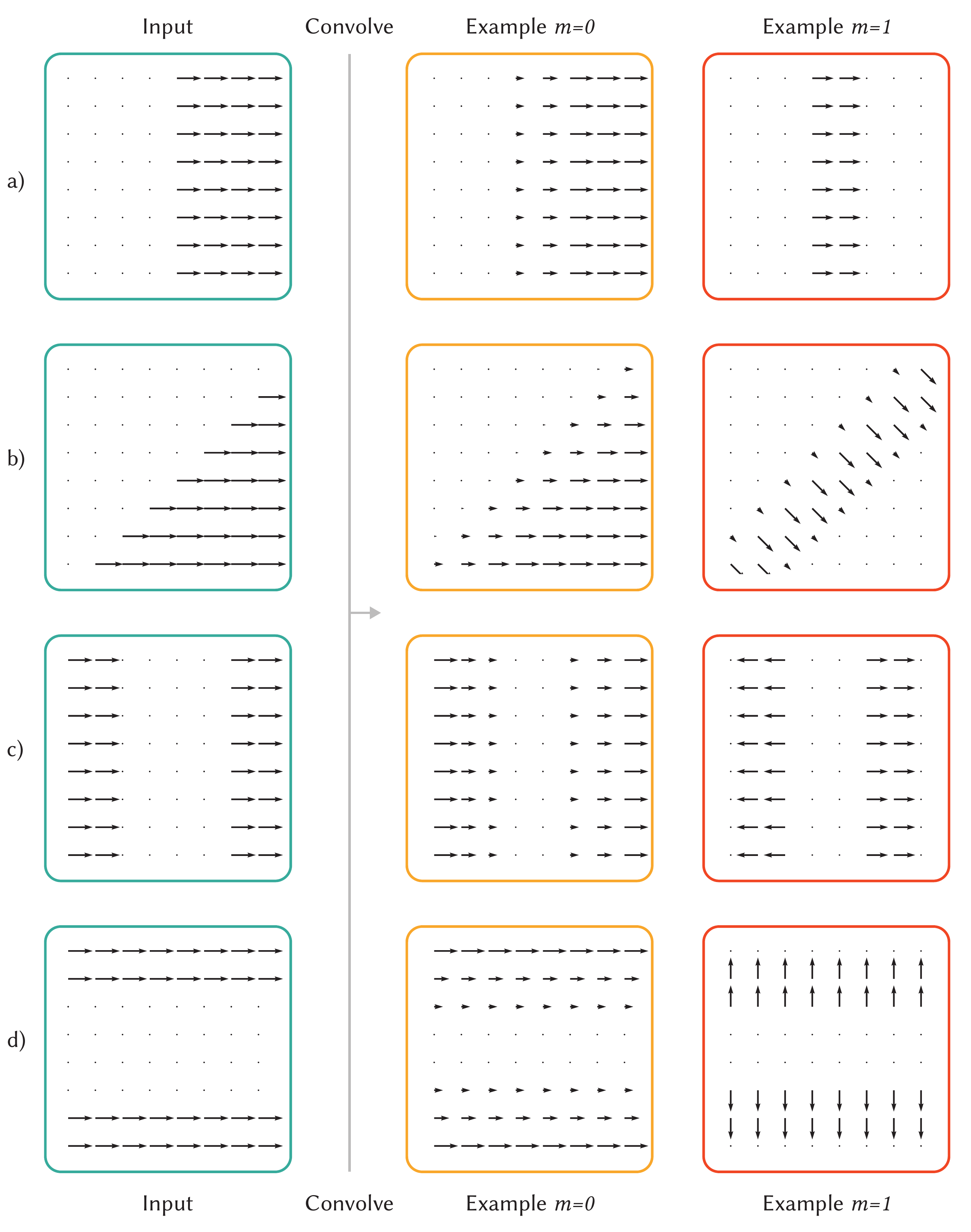}
    \Description{A grid of images with four rows and three columns. The first column shows input vector fields. These vector fields are convolved with the circular harmonic kernels of rotation order m=0 and m=1. The second and third column show the results for the m=0 and m=1 convolution, respectively. The inputs are the following: a) a vertical edge, b) a diagonal edge, c) two vertical edges, and d) two horizontal edges. The m=0 output smoothes the input vector field, the m=1 output shows vectors activating on the edge boundaries.}
    \caption{Examples of simple input vector fields (one feature) and the response of a convolution with $m=0$ and $m=1$ kernels.}
    \label{fig:examples}
\end{figure}

\paragraph*{Properties of HSNs}

In the following, we show that the result of convolutions
(\ref{eq:approxintegr}) as well as the layers of HSNs commute with
transformations of the coordinate systems in the tangent spaces. This means
that we can choose any coordinate system in the tangent spaces and use it
to build, train, and evaluate our networks. A different choice of the
coordinate systems leads to the same network, only with transformed features. 

As a consequence, HSNs do not suffer from the rotation ambiguity problem,
which we described in the introduction. The rotation ambiguity problem is caused by
the fact that due to the curvature of a surface, there is no consistent choice
of coordinate systems on a surface. So if a filter that is defined on $%
%TCIMACRO{\U{211d} }%
%BeginExpansion
\mathbb{R}
%EndExpansion
^{2}$ is to be applied everywhere on the surface, coordinate systems in the
tangent planes must be defined. Since there is no canonical choice of
coordinate systems, the coordinate systems at pairs of neighboring points are
not aligned. Due to the commutativity property (Lemma \ref{lem.commute}), a
convolution layer of our HSNs can use the features that are computed in
arbitrary coordinate systems and locally align them when computing the 
convolutions. 

In contrast, a network without that property cannot locally align the
features. If one would want to align the coordinate systems' features locally, they would
have to recompute the features starting from the first layer. Since this is
not feasible, one needs to work with non-aligned coordinate systems. 
The result is that the same kernel is applied with different coordinate systems at each location. Moreover, as features are combined in later layers, the network will learn from local relations of coordinate systems. This is undesirable, as these relations hold no meaningful information: they arise from arbitrary circumstances and choices of coordinate systems.

To prove that the features of HSNs commute with the coordinate changes,
this property must be shown for the individual operations. The non-linearity is invariant to changes of the basis as it only operates on the radial coordinates. Here, we restrict
ourselves to convolution, as for pooling, one can proceed analogous to this proof. 

\begin{lemma}
\label{lem.commute}The convolution (\ref{eq:approxintegr}) commutes with
changes of coordinate systems in the tangent planes.
\end{lemma}

\begin{proof}
We represent the coordinate system of a tangent space by specifying the
$x$-axis. Coordinates of points are then given by a complex number. If we
rotate the coordinate system at vertex $i$ by an angle of $-\phi$, the features
of order $M$ transform to
\begin{equation}
\mathbf{\tilde{x}}_{i,M}=e^{\ii M\phi}\mathbf{x}_{i,M},\label{eq.proof1}%
\end{equation}
where $\mathbf{\tilde{x}}_{i,M}$ is the feature in the rotated coordinate
system. The change of coordinate system also affects the polar coordinates of
any vertex $j$ in the tangent plane of $i$ and the transport of features from
any vertex $j$ to $i$
\begin{equation}
\tilde{\theta}_{ij}=\theta_{ij} + \phi \text{ \qquad\ }\tilde{P}_{j\rightarrow
i}(\mathbf{x}_{j,M})=e^{\ii M\phi}P_{j\rightarrow i}(\mathbf{x}%
_{j,M}),\label{eq.proof2}%
\end{equation}
where $\tilde{\theta}_{ij}$ is the angular coordinate of $j$ in the tangent
plane of $i$ with respect to the rotated coordinate system and $\tilde{P}$ is the
transport of features with respect to the rotated coordinate system. Then,
convolution with respect to the rotated coordinate system is given by
\begin{equation}
\mathbf{\tilde{x}}_{i,M+m}^{(l+1)}=\sum_{j\in\mathcal{N}_{i}}w_{j}\left(
R(r_{ij})e^{\ii(m\tilde{\theta}_{ij}+\beta)}\tilde{P}_{j\rightarrow
i}(\mathbf{x}_{j,M}^{(l)})\right)  .\label{eq.proof3}%
\end{equation}
Plugging (\ref{eq.proof2}) into (\ref{eq.proof3}), we get%
\begin{align}
\mathbf{\tilde{x}}_{i,M+m}^{(l+1)}  & =\sum_{j\in\mathcal{N}_{i}}w_{j}\left(
R(r_{ij})e^{\ii m\phi}e^{\ii(m\theta_{ij}+\beta)}e^{\ii M\phi}P_{j\rightarrow
i}(\mathbf{x}_{j,M}^{(l)})\right)  \label{eq.proof4}\\
& =e^{\ii(M+m)\phi}\mathbf{x}_{i,M+m}^{(l+1)}.
\end{align}
The latter agrees with the basis transformations of the features, see
(\ref{eq.proof1}). If the coordinate system at a vertex $j$ that is neighboring
$i$ is rotated by some angle, then this affects the transport of features
$P_{j\rightarrow i}$ from $j$ to $i$ and the feature $\mathbf{x}_{j,M}^{(l)}$.
However, since the transport and feature are rotated by the same angle but in
opposite directions, the term $P_{j\rightarrow i}(\mathbf{x}_{j,M}^{(l)})$ is
not affected. 
\end{proof}

\begin{figure*}[t]
    \centering
    \includegraphics[width=\textwidth]{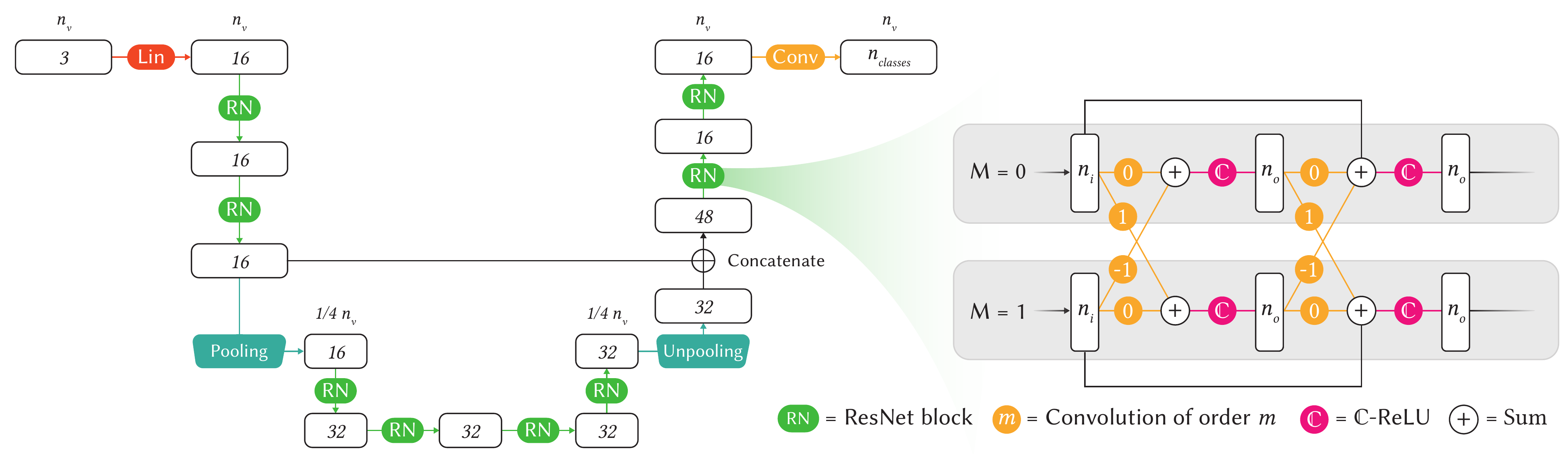}
    \Description{A schematic of the U-ResNet architecture and a detailed schematic of a ResNet Block.}
    \caption{HSN U-ResNet architecture used for correspondence and shape segmentation. On  on the left, the full architecture; on the right, a detail of the ResNet Block.}
    \label{fig:resnet}
\end{figure*}{}

\revised{We visualize the convolution of a simple $m=0$ feature with $m=0$ and $m=1$ filters in \autoref{fig:examples}. For each example, $\beta = 0$ and $R(r) = 1 - r$. The inputs are basic patterns: a) a vertical edge, b) a diagonal edge, c) two vertical edges, and d) two horizontal edges. We observe that $m=0$ convolutions smooth the signal and $m=1$ convolutions activate on edge boundaries.
We can also see \autoref{eq.proof1} at work: the input in a) and b) differs by a $45\degree$-rotation of the domain. The  $m=0$ outputs therefore are related by the same rotation of the domain. The equivariant $m=1$ outputs are related by a rotation of the domain and an additional rotation of the features. The same holds for c) and d), now with a rotation by $90\degree$.}

\paragraph*{Discretization}
In this paper, we detail Harmonic Surface Networks for meshes, as they are sparse and convenient representations of surfaces. Yet, we have attempted to formulate the building blocks for HSNs as general as possible. Thus, one could replace the use of the words `mesh' and `vertex' with `point cloud' and `point' and the method would remain largely the same. The main differences would be (1) how to compute the logarithmic map and parallel transport and (2) how to define a weighting per point. The first question has been answered by \cite{Sharp:2019:VHM}: the Vector Heat method can compute a logarithmic map for any discretization of a surface. The latter can be resolved with an appropriate integral approximation scheme.
%!TEX root = main.tex

\section{Experiments}\label{sect.experiments}
The goal of our experiments is twofold: we will substantiate our claims about improved performance with comparisons against state-of-the-art methods and we aim to evaluate properties of HSNs in isolation, in particular the benefits of the multi-stream architecture and rotation-equivariance.

\subsection{Implementation}
In the following paragraphs we discuss the architecture and data processing used in our experiments.

Following recent works in geometric deep learning \cite{verma2018feastnet, hanocka2019meshcnn, poulenard2018multi}, we employ a multi-scale architecture with residual connections. More specifically, we recreate the deep U-ResNet architecture from \cite{poulenard2018multi}. The U-ResNet architecture consists of four stacks of ResNet blocks, organized in a U-Net architecture with pooling and unpooling layers (\autoref{fig:resnet}). Each stack consists of two ResNet blocks, which, in turn, consist of two convolutional layers with a residual connection. We use 16 features in the first scale and 32 features in the second scale. Finally, we configure our network with two rotation order streams: $M=0$ and $M=1$, learning a rotation-equivariant kernel for every connection between streams.

\revised{For pooling operations, we sample a quarter of the points on the surface using farthest point sampling and take the average of the nearest neighbors. These neighbors are computed with the Heat Method  \cite{crane2013geodesics}, by diffusing indices from sampled points to all other points with a short diffusion time ($t = 0.0001$). With each pooling step, we grow the filter support by $1/\sqrt{\text{ratio}}$. Thus, we grow the radius with a factor $2$.
At the unpooling stage, we propagate the features from the points sampled in the pooling stage to the nearest neighbors, using parallel transport.}
We use ADAM \cite{Kingma2014AdamAM} to minimize the negative log-likelihood and train for 100 epochs on correspondence and 50 epochs for shape segmentation and mesh classification.

For each experiment, we normalize the scale of each shape, such that the total surface area is equal to 1. \revised{We then compute local supports for each vertex $i$ by assigning all points within a geodesic disc of radius $\epsilon$ to its neighborhood $\mathcal{N}_i$. We normalize the weighting described in \autoref{eq:weighting} by the sum of weights for each neighborhood, to compensate for the different sizes of support of the filters in the different layers, in particular, after pooling and unpooling.}

Next, we employ the Vector Heat Method out of the box \cite{Sharp:2019:VHM, geometrycentral} to compute the logarithmic map and parallel transport for each neighborhood. For our multi-scale architecture, we first compute local supports for each scale and coalesce the resulting graph structures into one multi-scale graph. This way, we can precompute the necessary logarithmic maps in one pass.
Wherever we can, we use raw xyz-coordinates as input to our network. To increase the robustness of the network against transformations of the test set, we randomly scale and rotate each shape with a factor sampled from $\mathcal{U}(0.85, 1.15)$ and $\mathcal{U}(-\frac{1}{8}\pi, \frac{1}{8}\pi)$, respectively. For correspondence, we use SHOT descriptors, to provide a clean comparison against previous methods.

The networks are implemented using PyTorch Geometric \cite{FeyLenssen2019}. The implementation can be retrieved from \\ \url{https://github.com/rubenwiersma/hsn}.

\subsection{Comparisons}
Following from the benefits outlined in the introduction, we expect HSNs to show improved results over existing spatial methods, even though fewer parameters are used for each kernel. We experimentally validate these expected benefits by applying HSN on three tasks: shape classification on SHREC \cite{lian2011}, correspondence on FAUST \cite{Bogo:CVPR:2014}, and shape segmentation on the human segmentation dataset proposed by \cite{maron2017convolutional}.

\begin{table}[]
    \caption{Results for shape classification on 10 training samples per class of HSN against previous work.}
    \label{tab:shapeclassification}
    \begin{tabular}{ll}
    Method              & Accuracy \\ \hline \hline
    \multicolumn{1}{l|}{HSN (ours)}         & \textbf{\revised{96.1\%}}   \\
    \hline
    \multicolumn{1}{l|}{MeshCNN}     & 91.0\%  \\
    \multicolumn{1}{l|}{GWCNN}        & 90.3\%  \\
    \multicolumn{1}{l|}{GI} & 88.6\%  \\
    \multicolumn{1}{l|}{\revised{MDGCNN}} & \revised{82.2\%} \\
    \multicolumn{1}{l|}{\revised{GCNN}} & \revised{73.9\%} \\
    \multicolumn{1}{l|}{SG} & 62.6\% \\
    \multicolumn{1}{l|}{\revised{ACNN}} & \revised{60.8\%} \\
    \multicolumn{1}{l|}{SN}  & 52.7\%  \\
    \end{tabular}
\end{table}

\begin{figure}[b]
    \centering
    \includegraphics[width=\columnwidth]{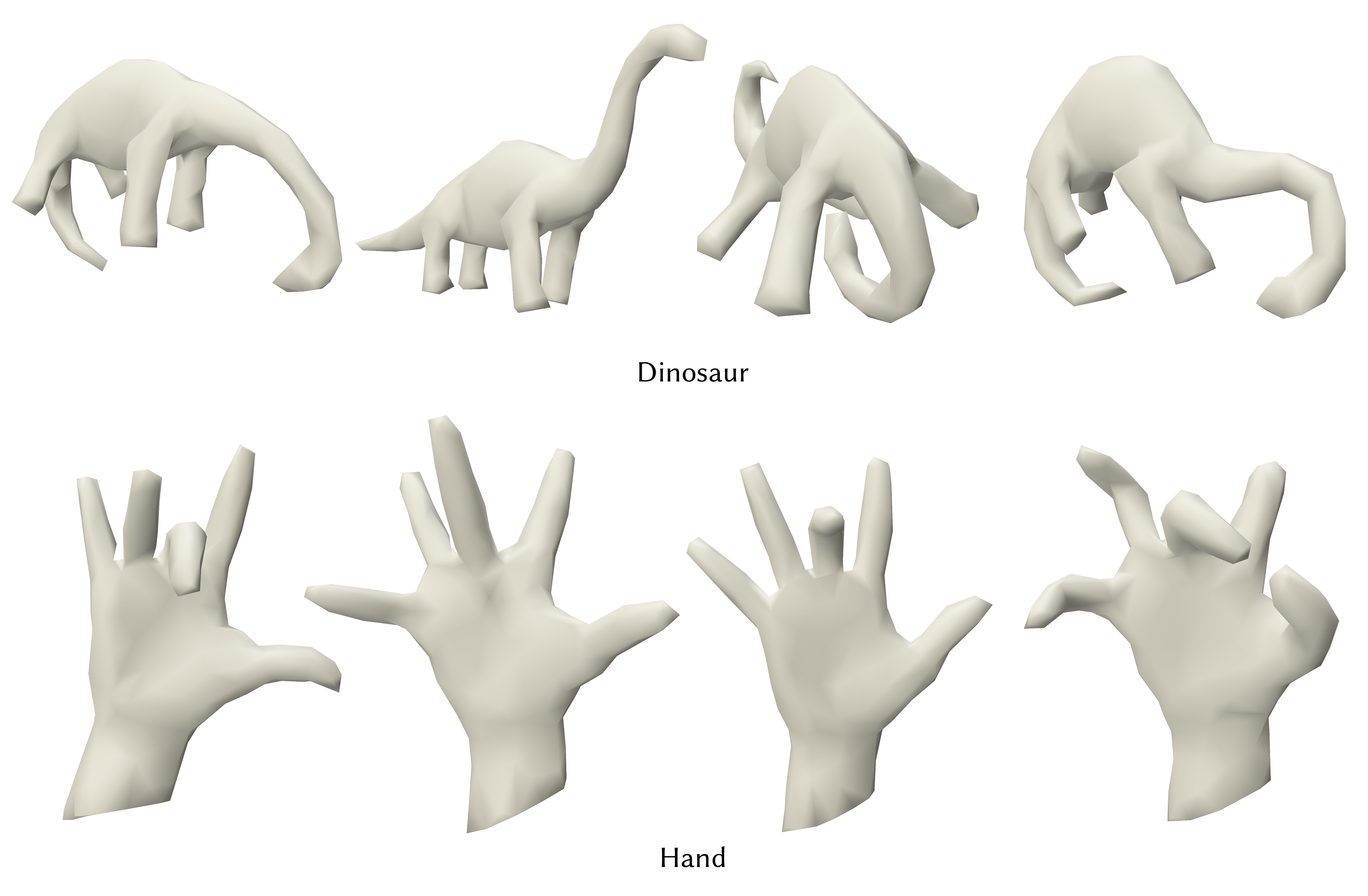}
    \Description{Two rows of shapes. The first row shows four deformed meshes of a dinosaur. The second row shows four deformed meshes of a hand.}
    \caption{Two example classes with four shapes each from the SHREC shape classification dataset.}
    \label{fig:classificationex}
\end{figure}

\paragraph{Shape classification}
We train an HSN to perform classification on the SHREC dataset \cite{lian2011}, consisting of 30 classes (\autoref{fig:classificationex}). We only train on 10 training samples per class. Like \cite{hanocka2019meshcnn}, we take multiple random samplings from the dataset to create these training sets and average over the results. We train for 50 epochs, compared to the 200 epochs used in previous works. \revised{This task is challenging due to the low number of training samples and labels. Therefore, we reduce the size of our network and consequently, the number of parameters to learn. We only use the first half of the U-ResNet architecture, with only one ResNet block per scale.} To obtain a classification, we retrieve the radial components from the last convolutional layer, followed by a global mean pool. The initial radius of our filters is $\epsilon=0.2$ and the number of rings in the radial profile is 6.

For our comparisons, we cite the results from \cite{hanocka2019meshcnn} and \cite{10.1111/cgf.13244}, comparing our method to MeshCNN \cite{hanocka2019meshcnn}, SG \cite{10.1145/1899404.1899405}, SN \cite{wu20153d}, GI \cite{10.1007/978-3-319-46466-4_14}, and GWCNN \cite{10.1111/cgf.13244}. \revised{Additionally, we train MDGCNN \cite{poulenard2018multi}, GCNN \cite{masci2015geodesic}, and ACNN \cite{boscaini2016learning} with the exact same architecture as HSN.} The results are outlined in \autoref{tab:shapeclassification}. HSN outperforms all previous methods, demonstrating the effectiveness of our approach, even for lower training times. \revised{One explanation for the large gap in performance is the low number of training samples. HSN uses fewer parameters, resulting in fewer required training samples. This is supported by results in \cite{worrall2017harmonic}, who also show higher performance for small datasets compared to other methods. The low number of samples also explains why some non-learning methods outperform learning methods. An additional problem faced by ACNN is the low quality of the meshes, resulting in a disparate principal curvature field. This obstructs the network from correctly relating features at neighboring locations and degrades performance. This same effect is observed when applying HSN without parallel transport, aligned to principal curvature directions (\autoref{tab:classification_m0m1}).}

\begin{table}[]
    \caption{Results for shape segmentation by HSN and related methods.}
    \label{tab:shapesegmentation}
    \begin{tabular}{lll}
    Method                           & \# Features             & Accuracy \\ \hline \hline
    \multicolumn{1}{l|}{HSN (ours)}         & \multicolumn{1}{l|}{3}  & \revised{91.14\%}   \\
    \hline
    \multicolumn{1}{l|}{MeshCNN}     & \multicolumn{1}{l|}{5}  & \textbf{92.30\%}  \\
    \multicolumn{1}{l|}{SNGC}        & \multicolumn{1}{l|}{3}  & 91.02\%  \\
    \multicolumn{1}{l|}{PointNet++}  & \multicolumn{1}{l|}{3}  & 90.77\%  \\
    \multicolumn{1}{l|}{MDGCNN}      & \multicolumn{1}{l|}{64} & 89.47\%  \\
    \multicolumn{1}{l|}{Toric Cover} & \multicolumn{1}{l|}{26} & 88.00\%  \\
    \multicolumn{1}{l|}{DynGraphCNN} & \multicolumn{1}{l|}{64} & 86.40\%  \\
    \multicolumn{1}{l|}{GCNN}        & \multicolumn{1}{l|}{64} & 86.40\%  \\
    \multicolumn{1}{l|}{\revised{ACNN}}       & \multicolumn{1}{l|}{3} & \revised{83.66}\% \\
    \end{tabular}
\end{table}

\begin{figure}[b]
    \includegraphics[width=\columnwidth]{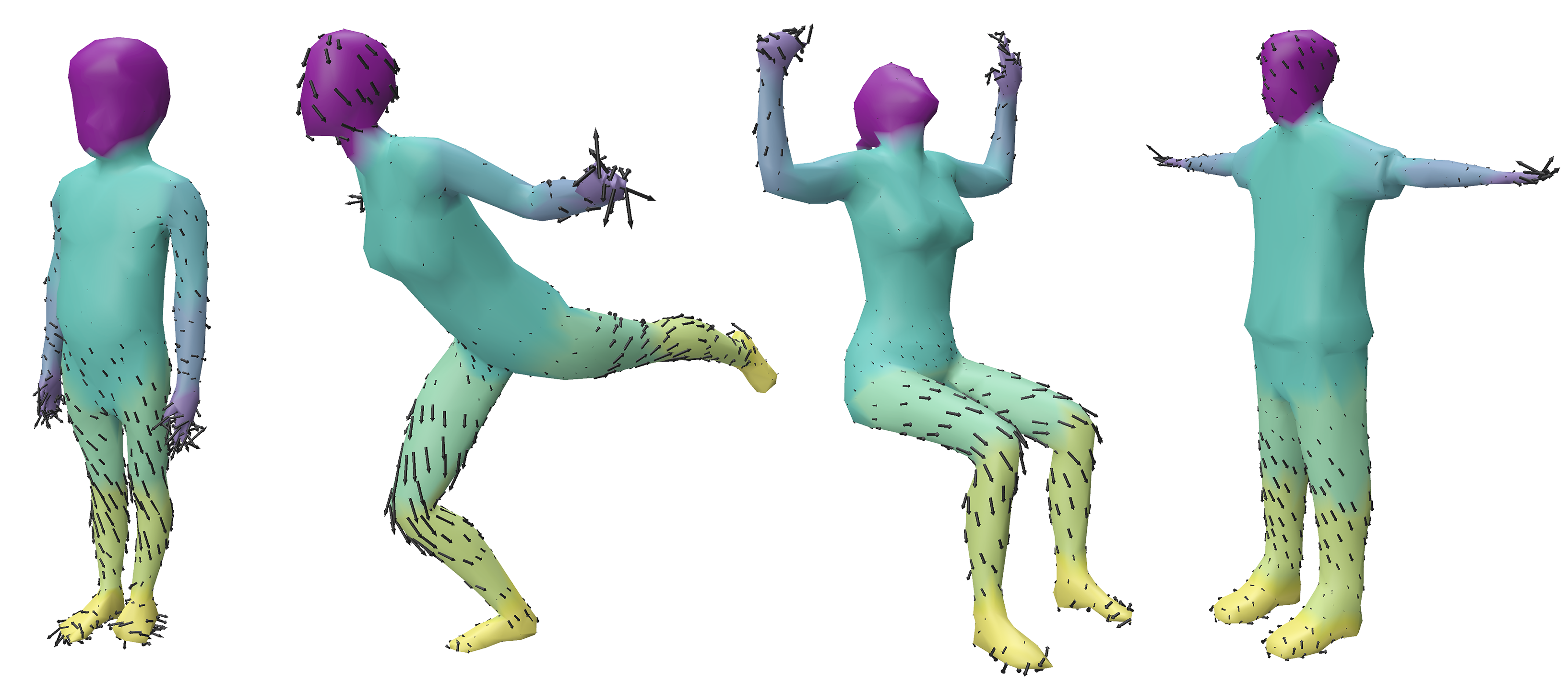}
    \Description{Four human shapes with coloring based on the body part. The meshes show a vector field with large vectors on the legs and hands.}
    \caption{Vector-valued featuremap and label predictions from our Harmonic Surface Network trained on shape segmentation.}
    \label{fig:hsnfeatures}
\end{figure}

\paragraph{Shape segmentation}
Next, we demonstrate HSN on shape segmentation. We train our network to predict a body-part annotation for each point on the mesh. We evaluate our method on the dataset proposed by Maron et al. \shortcite{maron2017convolutional}, which consists of shapes from FAUST ($n=100$) \cite{Bogo:CVPR:2014}, SCAPE ($n=71$) \cite{10.1145/1073204.1073207}, Adobe Mixamo ($n=41$) \cite{adobe_2016}, and MIT ($n=169$) \cite{10.1145/1360612.1360696} for training and shapes from the human category of SHREC07 ($n=18$) for testing. \revised{The variety in sources provides a variety in mesh structures as well, which tests HSN's robustness to changes in mesh structure.}

We use the U-ResNet architecture, providing xyz-coordinates as input and evaluate the class prediction directly from the final convolutional layer. The logarithmic map and parallel transport are computed using the original meshes from \cite{maron2017convolutional}. To limit the training time, 1024 vertices are sampled on each mesh using farthest point sampling to be used in training and testing. This type of downsampling is also applied by \cite{hanocka2019meshcnn} (to 750 vertices) for similar reasons. The initial \revised{support} of our \revised{kernels} is $\epsilon=0.2$ and the number of rings in the radial profile is 6.

We report the accuracy as the fraction of vertices that was classified correctly across all test shapes. For the comparisons, we cite the results from \cite{poulenard2018multi}, \cite{DBLP:journals/corr/abs-1812-10705} and \cite{hanocka2019meshcnn}. HSN produces competitive results compared to state-of-the-art methods, performing only slightly worse than MeshCNN.

To understand the complex features learned by the network, we visualize the features on the model, alongside our predictions for segmentation. We can interpret the complex features as intrinsic vectors on the surface and visualize them as such using Polyscope \cite{polyscope}. \autoref{fig:hsnfeatures} shows a single feature in the $M=1$ stream from the second-to-last layer. We observe high activations on certain bodyparts (legs, hands) and note that the intrinsic vectors are aligned. Our interpretation is that the corresponding filter `detects' legs and hands. The alignment of features is beneficial to the propagation of these activations through the network; if features are oriented in opposing directions, they can cancel each other out when summed.
%removed: are more likely to cancel

\begin{figure}[t]
    \centering
    \includegraphics[width=\columnwidth]{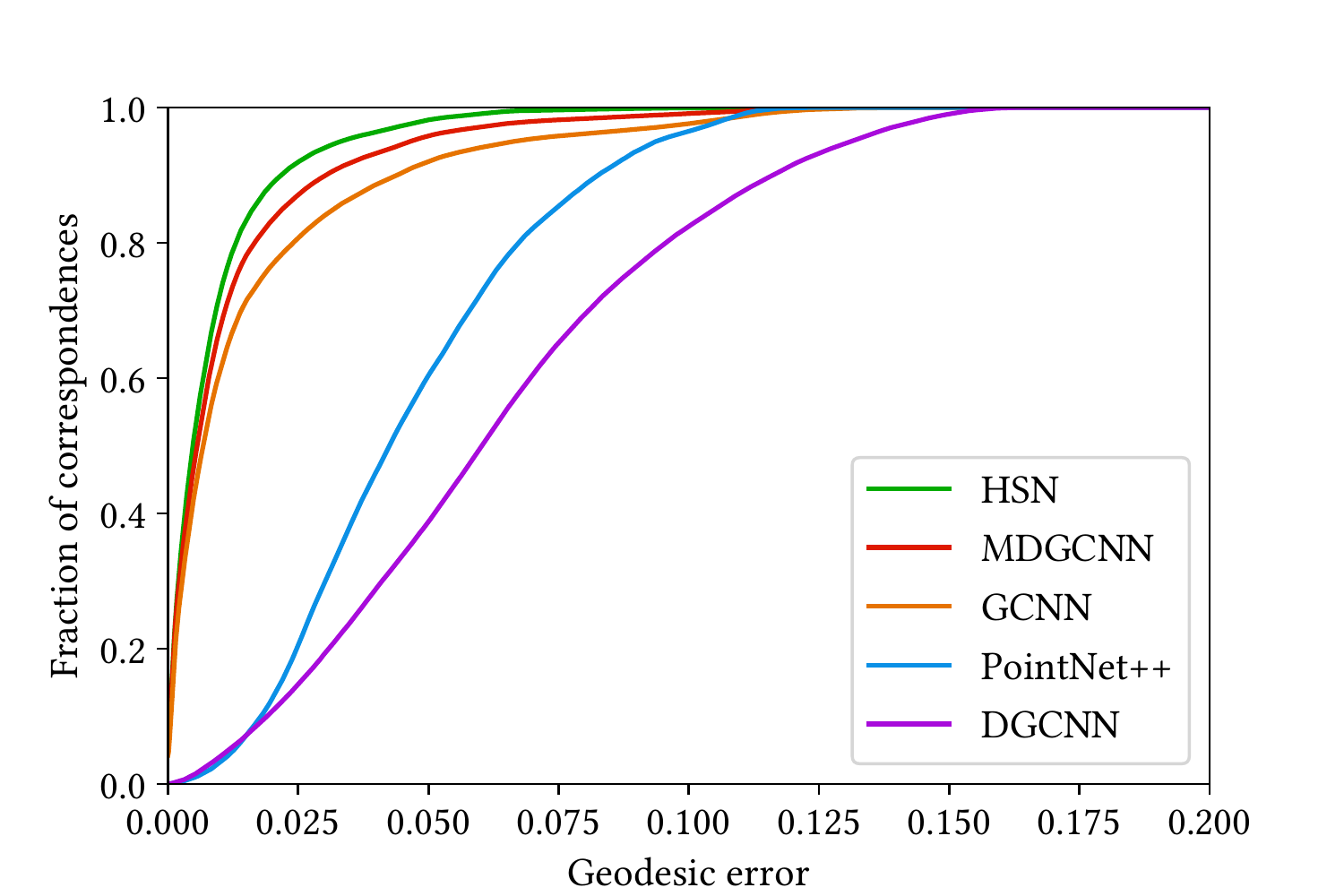}
    \Description{Five curves, all starting with a fraction of correct correspondences near 0. HSN, MDGCNN, GCNN quickly rise to about 0.9 at 0.025 geodesic error and slowly converge to 1. HSN has the highest curve, signifying better performance. PointNet++ and DGCNN have S-like curves that converge to 1 at 0.1 and 0.15 geodesic error, respectively.}
    \caption{Fraction of correspondences for a given geodesic error on the remeshed FAUST dataset using HSN (ours), MDGCNN, GCNN, PointNet++, and DGCNN.}
    \label{fig:accuracy_remeshed}
\end{figure}

\paragraph{Correspondence}
Correspondence finds matching points between two similar shapes. We evaluate correspondence on the widely used FAUST dataset \cite{Bogo:CVPR:2014}. FAUST consists of 100 scans of human bodies in 10 different poses with ground-truth correspondences. We set up the network to predict the index of the corresponding vertex on the ground-truth shape. To this end, we add two fully connected layers (FC256, FC$N_{\text{vert}}$) after the U-ResNet architecture. The initial radius of our filters is $\epsilon=0.1$ and the number of rings in the radial profile is 2.

\begin{figure}[t]
    \centering
    \includegraphics[width=\columnwidth]{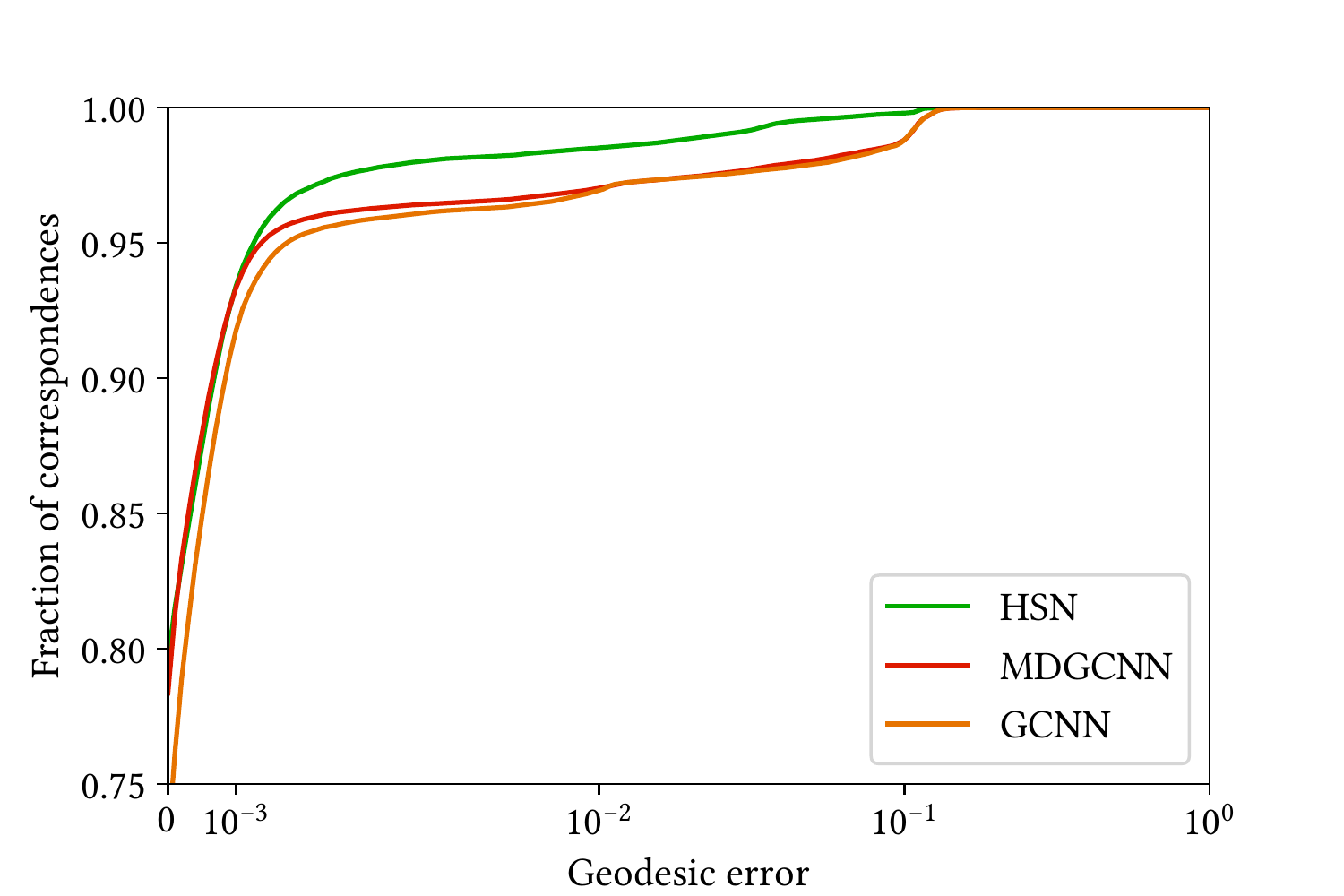}
    \Description{Three curves, starting with a fraction of correct correspondences near 0.8. HSN and MDGCNN follow a similar trajectory until about 0.001 geodesic error. HSN rises above MDGCNN and GCNN to about 0.98 and slowly converges to 1.}
    \caption{Fraction of correspondences for a given geodesic error on the original FAUST dataset using HSN (ours), MDGCNN, and GCNN. Note: the x axis is set to logarithmic scale.}
    \label{fig:accuracy_original}
\end{figure}

\begin{figure}[b]
    \includegraphics[width=\columnwidth]{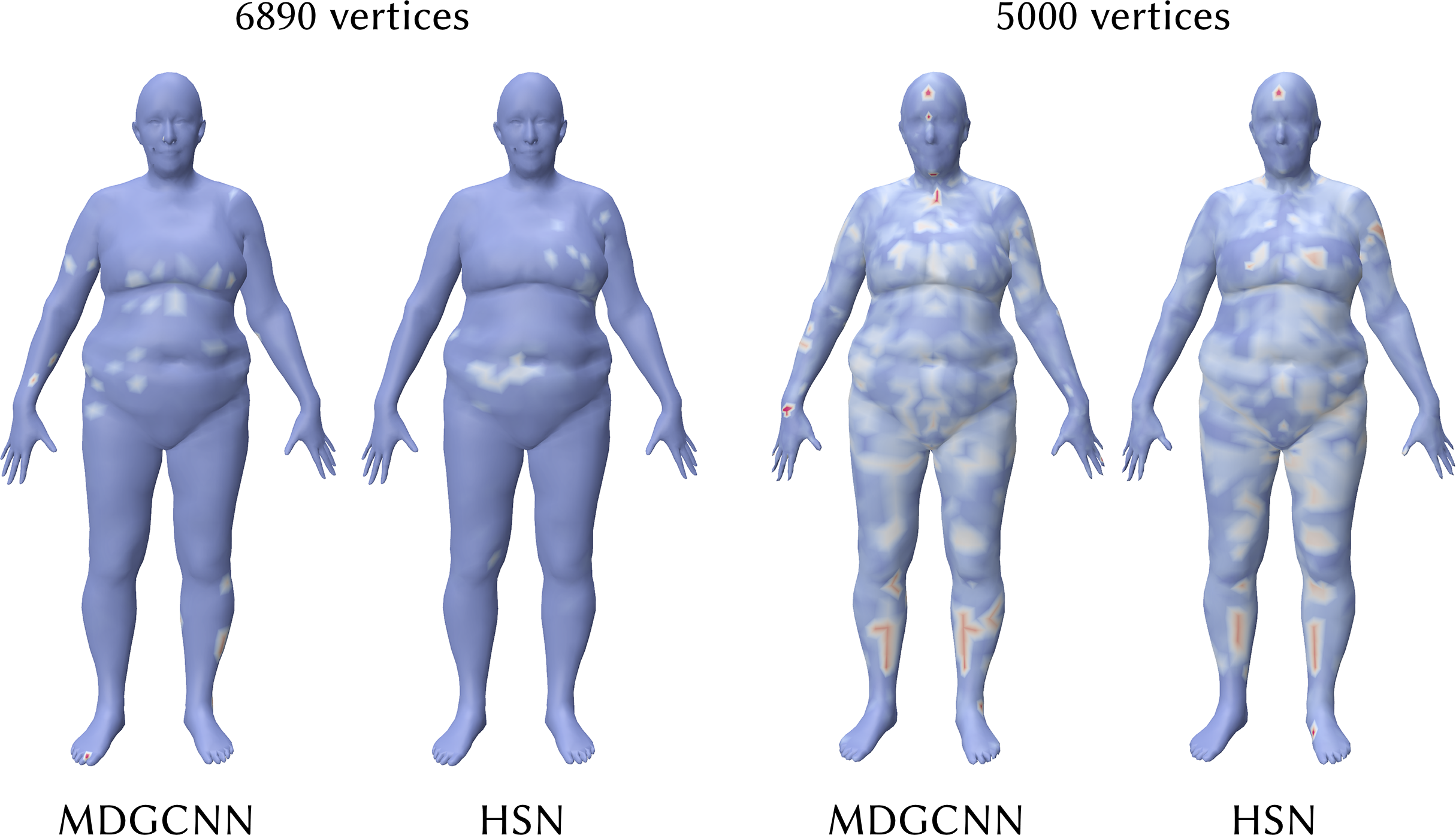}
    \Description{Four meshes of human bodies with error heatmaps. The first two shapes show results for MDGCNN and HSN on the original FAUST dataset, the last two shapes show results for MDGCNN and HSN on the remeshed FAUST dataset. The heatmaps for HSN show fewer highlighted areas, signifying smaller errors.}
    \caption{Geodesic error visualised on the test shapes, shown: the first test shape for MDGCNN and HSN, for both the original and remeshed dataset.}
    \label{fig:errormap}
\end{figure}

We train on the first 80 shapes and report the fraction of correct correspondences for a given geodesic error on the last 20 models (Figures \ref{fig:accuracy_remeshed} and \ref{fig:accuracy_original}). As input to our network, we provide 64-dimensional SHOT descriptors, computed for $12\%$ of the shape area. Similar to Poulenard and Ovsjanikov \shortcite{poulenard2018multi}, we train our network on a remeshed version of the FAUST dataset as well, since the original FAUST dataset exhibits the same connectivity between shapes. The remeshed dataset is a challenge that is more representative of real applications, where we cannot make any assumptions about the connectivity or discretization of our input. Having recreated the same architecture, with the same input features, we can fairly compare our method to MDGCNN and report the results obtained by Poulenard and Ovsjanikov \shortcite{poulenard2018multi}.

The results in Figures \ref{fig:accuracy_remeshed}, \ref{fig:accuracy_original}, and \ref{fig:errormap} show that HSN outperforms the compared state-of-the-art methods. More importantly, HSN improves existing results on the remeshed FAUST dataset, demonstrating the robustness of the method to changes in the discretization of the surface.

\paragraph{Parameter count and memory usage}
The following calculation shows that HSN achieves this performance using fewer parameters and with less impact on memory during computation. Let $n_i$ and $n_o$ be the number of input and output features, $n_{\rho}$ and $n_{\theta}$ the number of rings and angles on the polar grid, and $n_m$ the number of rotation order streams. MDGCNN \cite{poulenard2018multi} and GCNN \cite{masci2015geodesic} learn a weight matrix for every location on a polar grid, resulting in a parameter complexity of $O(n_i\,n_o\,n_{\rho}\,n_{\theta})$. HSN learns the same weight matrix, but only for the radial profile and the phase offset. We chose to learn these weights separately for every connection between streams, resulting in a parameter complexity of $O(n_i\,n_o\,(n_{\rho} + 1)\,n_m^2)$. If one chooses to learn weights per stream, which the original H-Nets opt for, this is reduced to $O(n_i\,n_o\,(n_{\rho} + 1)\,n_m)$. Removing the dependency on $n_{\theta}$ has a high impact on the number of parameters: for $n_{\rho} = 2$ and $n_{\theta} = 8$, HSN uses $75\%$ of the parameters used by MDGCNN \cite{poulenard2018multi} (122880 vs. 163840), only considering the convolutional layers. If we were to use the same number of rings and angles as used by GCNN \cite{masci2015geodesic}, $n_{\rho} = 5$ and $n_{\theta} = 16$, HSN would use only $30\%$ of the parameters used by other spatial methods.

Concerning space complexity, MDGCNN \cite{poulenard2018multi} stores the result from every rotation throughout the network, multiplying the memory complexity of the model by the number of directions used, which tends to be between 8 and 16. In comparison, HSN's complex-valued features only increase the memory consumption for storing the separate streams and complex features. For two streams, this increases the space complexity by a factor of 4. This is important for the ability to scale our method to higher-resolution shapes and larger datasets, necessary for use in applications.

\begin{figure}[b]
    \centering
    \begin{subfigure}[b]{0.18\columnwidth}
        \includegraphics[width=\textwidth]{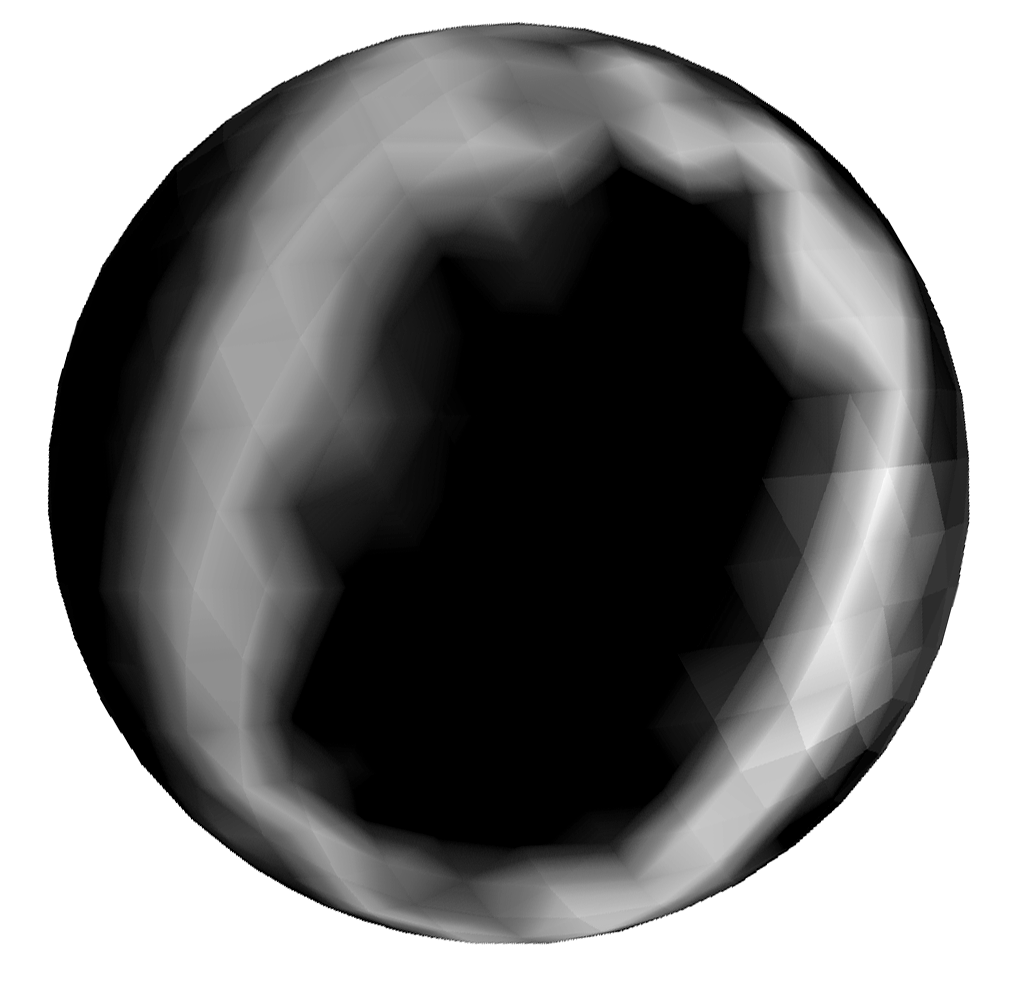}
    \end{subfigure}
    \begin{subfigure}[b]{0.18\columnwidth}
        \includegraphics[width=\textwidth]{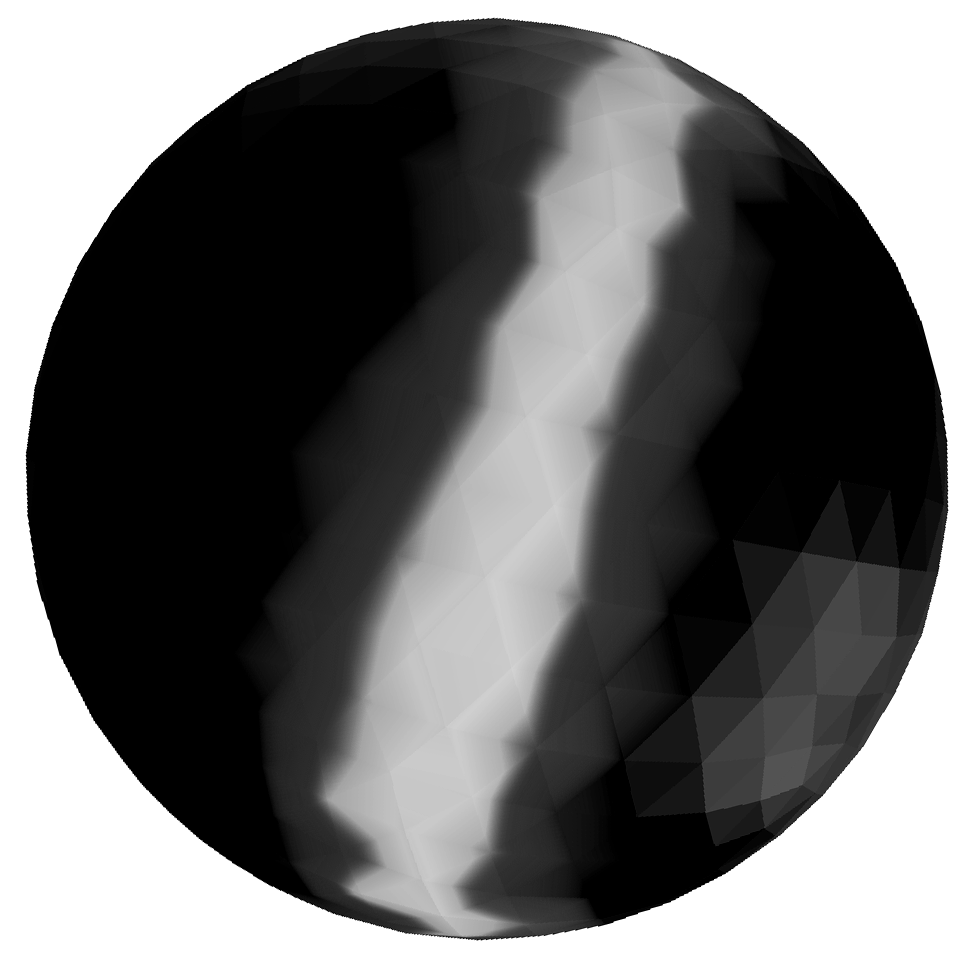}
    \end{subfigure}
    \begin{subfigure}[b]{0.18\columnwidth}
        \includegraphics[width=\textwidth]{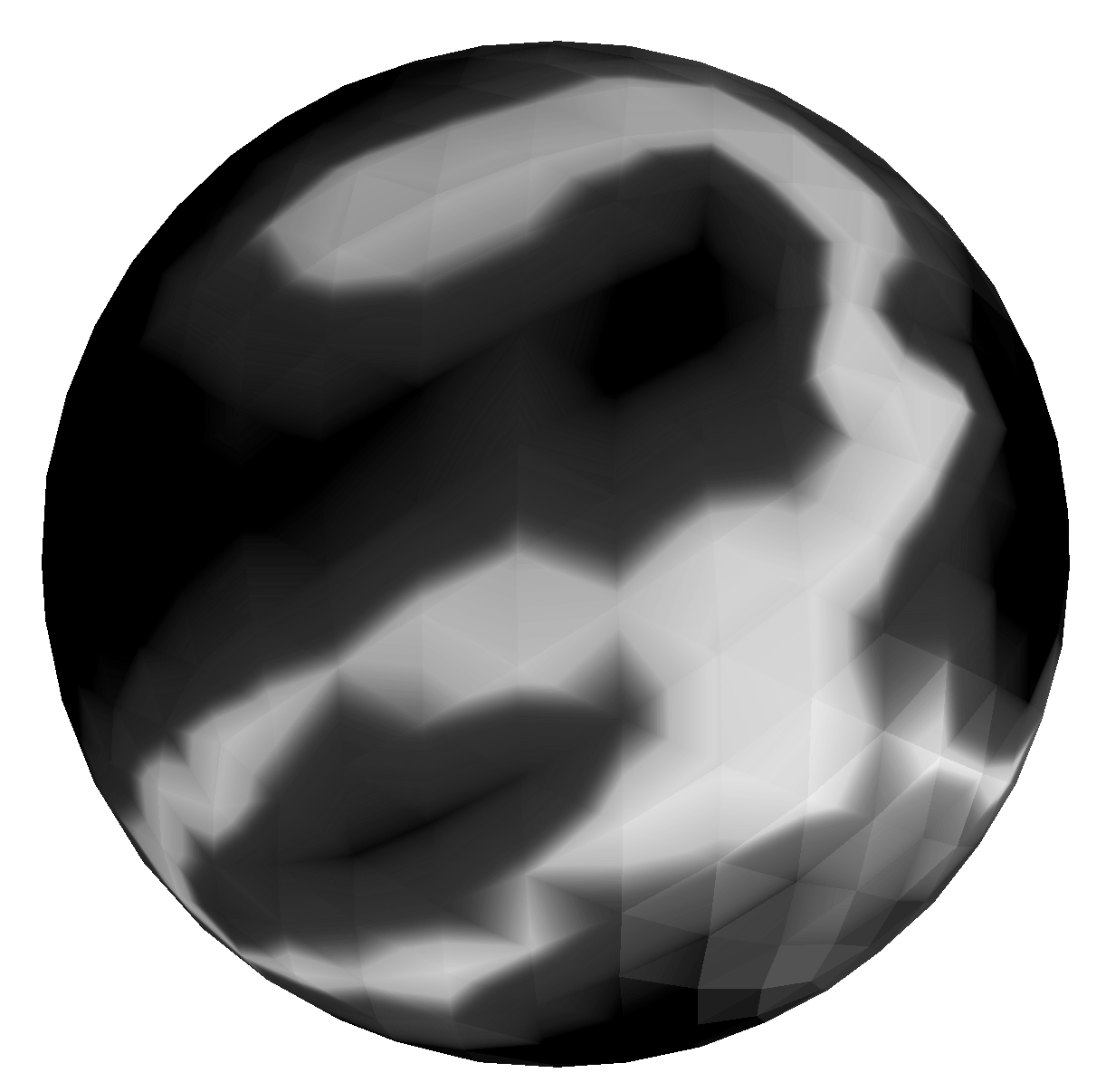}
    \end{subfigure}
    \begin{subfigure}[b]{0.18\columnwidth}
        \includegraphics[width=\textwidth]{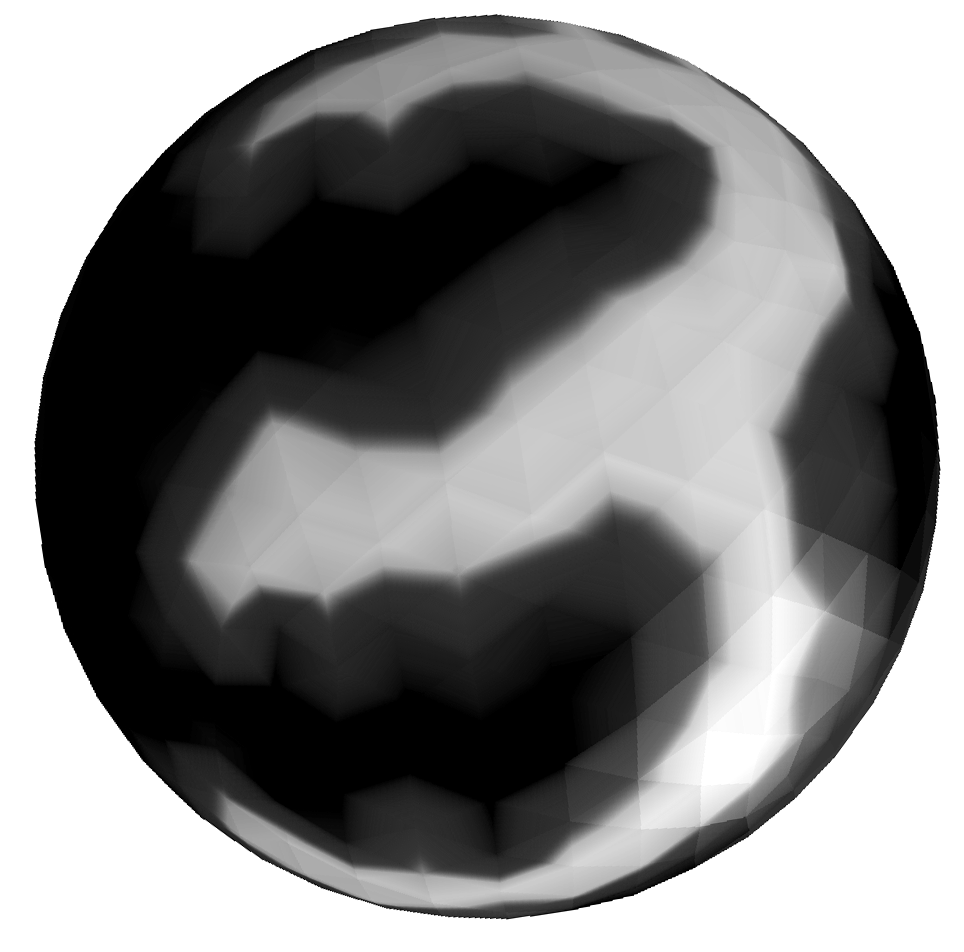}
    \end{subfigure}
    \begin{subfigure}[b]{0.18\columnwidth}
        \includegraphics[width=\textwidth]{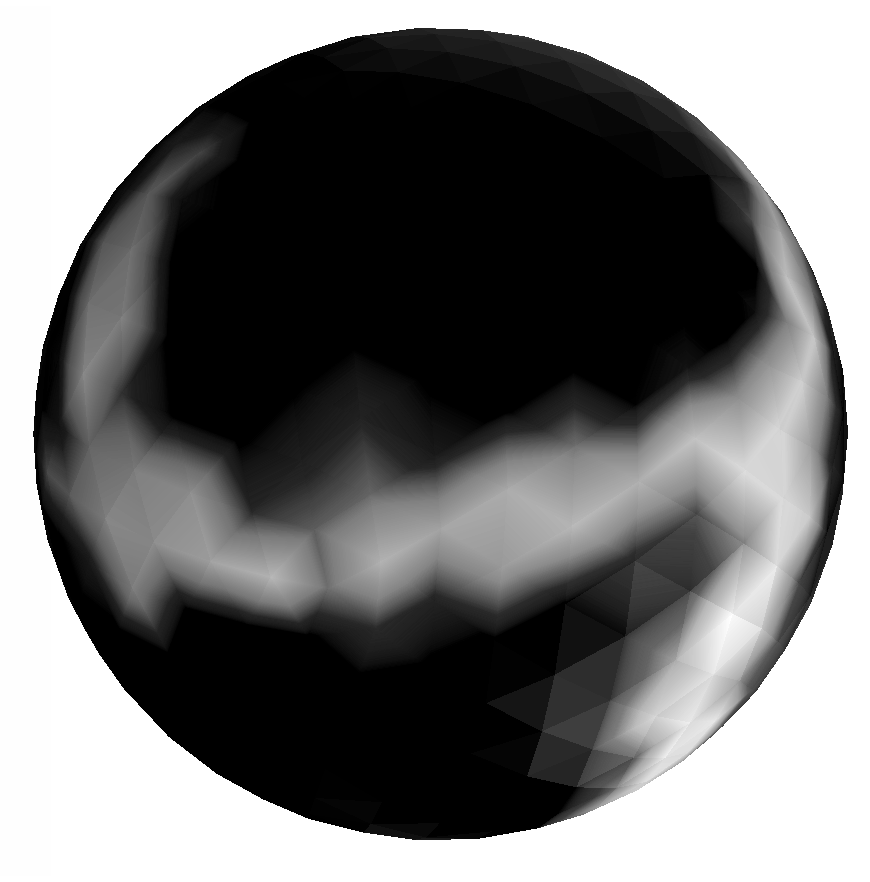}
    \end{subfigure}
    \Description{Five spheres, textured with numbers from MNIST.}
    \caption{Rotated MNIST mapped to a sphere}\label{fig:mnistsphere}
\end{figure}

\subsection{Evaluation}
Aside from comparisons with other methods, we intend to get a deeper understanding of the properties of HSN; specifically, the benefit of the $M=1$ stream w.r.t. rotation-invariant methods and a single $M=0$ stream.

\begin{figure}[t]
    \centering
    \includegraphics[width=\columnwidth]{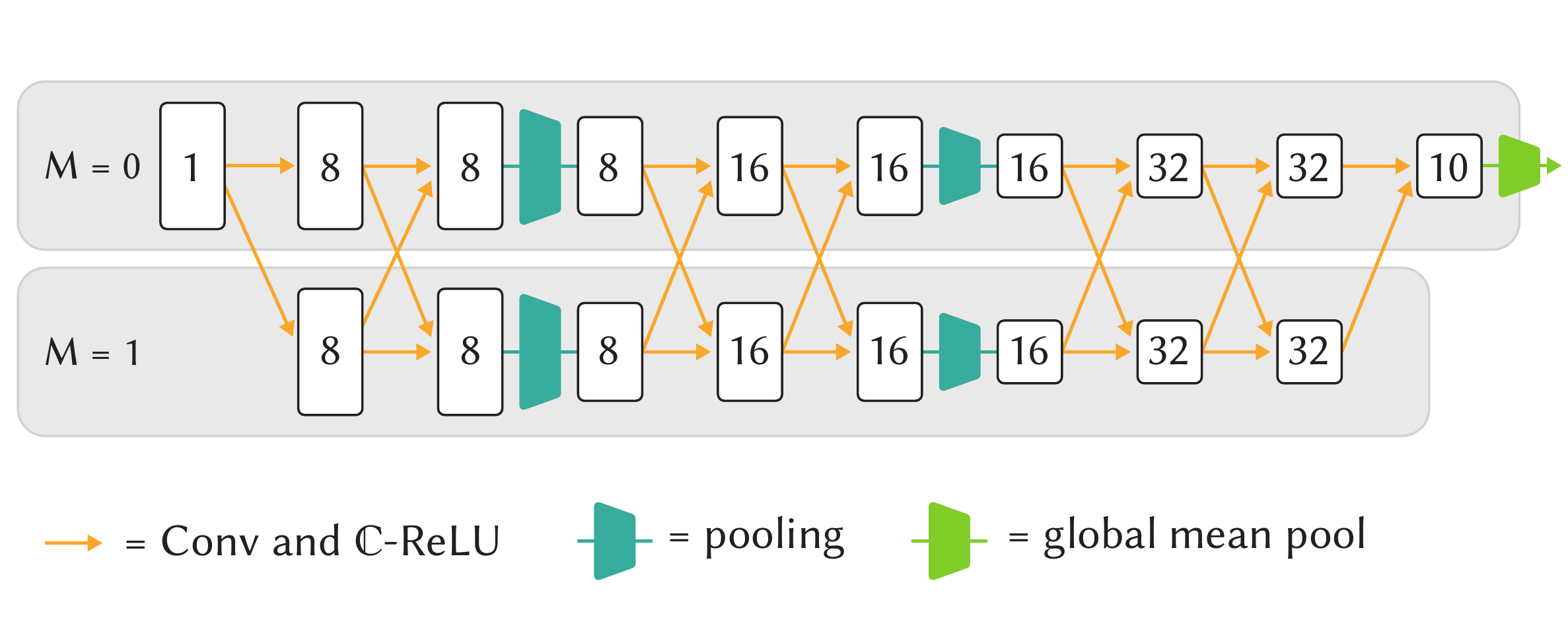}
    \Description{A schematic representation of a neural network architecture.}
    \caption{Architecture for classification of Rotated MNIST.}
    \label{fig:rmnistarch}
\end{figure}

To evaluate the benefit of different rotation order streams in our method, we compare two different stream configurations of our method: a single-stream architecture ($M=0$) and a double-stream architecture ($M=0$ and $M=1$). We limit this experiment to only these two configurations, because (1) experiments from \cite{worrall2017harmonic} demonstrate that rotation order streams of $M>1$ do not significantly improve the performance of the network and (2) the $M=0$ and $M=1$ stream features have an intuitive interpretation on the surface as scalar values and intrinsic vectors, respectively.

We evaluate the two configurations on a new task: classification of digits from the Rotated MNIST mapped to a sphere, as well as the shape segmentation and classification tasks from the comparison experiments.

\paragraph{Rotated MNIST on a sphere.} We use an elliptical mapping \cite{fong2015analytical} to map the grayscale values from the images of the Rotated MNIST dataset \cite{larochelle2007empirical} to the vertices of a unit sphere with 642 vertices. The Rotated MNIST dataset consists of 10000 randomly rotated training samples, 2000 validation samples, and 50000 test images separated in 10 classes. We use an architecture similar to the one in \cite{worrall2017harmonic}: Conv8, Conv8, Pool0.5, Conv16, Conv16, Pool0.25, Conv32, Conv32, Conv10 (\autoref{fig:rmnistarch}). The kernel supports for each scale are the following: 0.3, 0.45, 0.8 (the geodesic diameter of the unit sphere is $\pi$).

The two stream architecture demonstrates significant improvement over the single-stream architecture with a higher parameter count to compensate for using only one stream: 94.10\% vs. 75.57\% (\autoref{tab:mnistsphere_results}). This affirms the benefit of rotation-equivariant streams for learning signals on surfaces.

\begin{table}[t]
\caption{Results of HSN tested on Rotated MNIST mapped to a sphere for a single- and double-stream configuration.}
\label{tab:mnistsphere_results}
\begin{tabular}{lll}
Method                           & Streams ($M=\ldots$)             & Accuracy \\ \hline \hline
\multicolumn{1}{l|}{HSN}     & \multicolumn{1}{l|}{0, 1}  & \textbf{94.10\%} \\
\hline
\multicolumn{1}{l|}{HSN}         & \multicolumn{1}{l|}{0}  & \revised{70.68}\% \\
\multicolumn{1}{l|}{HSN (parameters x4)}         & \multicolumn{1}{l|}{0}  & \revised{75.57}\%   \\
\end{tabular}
\end{table}

\paragraph{Shape Segmentation and Classification}
\revised{We repeat our experiments for shape segmentation and classification in four different configurations: the double-stream configuration used in section 5.2; a single-stream configuration; a single-stream configuration with four times the parameters; and the double stream configuration aligned to a smoothed principal curvature field, instead of local alignment with parallel transport. The single-stream configurations aim to provide insight into the benefit of the rotation-equivariant stream and to rule out the sheer increase in parameters as the cause of this performance boost. The last configuration shows the benefit of locally aligning rotation-equivariant features over aligning kernels to smoothed principal curvature fields, as is done by ACNN \cite{boscaini2016learning}.}

\revised{The results in \autoref{tab:shapeseg_m0m1} and \autoref{tab:classification_m0m1} show that the double-stream architecture improves the performance from satisfactory to competitive. Furthermore, an increase in parameters has a negative impact on performance, likely due to the relatively low number of training samples. This becomes even more apparent when comparing the learning curves for each configuration in \autoref{fig:shapeseg_m0m1}: the double-stream configuration is more stable and performs better than both single-stream configurations. These results support the benefit of the rotation-equivariant stream.

Finally, we find that HSN performs significantly better when using parallel transport than when aligned with a smoothed principal curvature direction (pc aligned): for shape segmentation, the benefit introduced by the $M=1$ stream is diminished, and for shape classification, we observe a large drop in performance, likely induced by the coarseness of the meshes and the resulting low-quality principal curvature fields.}

\begin{table}[t]
    \caption{Results of HSN tested on shape segmentation for multiple configurations.}
    \label{tab:shapeseg_m0m1}
    \begin{tabular}{lll}
    Method                           & Streams ($M=\ldots$)             & Accuracy \\ \hline \hline
    \multicolumn{1}{l|}{HSN}     & \multicolumn{1}{l|}{0, 1}  & \textbf{\revised{91.14}\%} \\
    \hline
    \multicolumn{1}{l|}{\revised{HSN}}         & \multicolumn{1}{l|}{0}  & \revised{88.74}\%   \\
    \multicolumn{1}{l|}{HSN (parameters $\times 4$)}         & \multicolumn{1}{l|}{0}  & \revised{87.25}\%   \\
    \multicolumn{1}{l|}{HSN (pc aligned)}     & \multicolumn{1}{l|}{0, 1}  & \revised{86.22}\%
    \end{tabular}
\end{table}

\begin{figure}[t]
    \centering
    \includegraphics[width=\columnwidth]{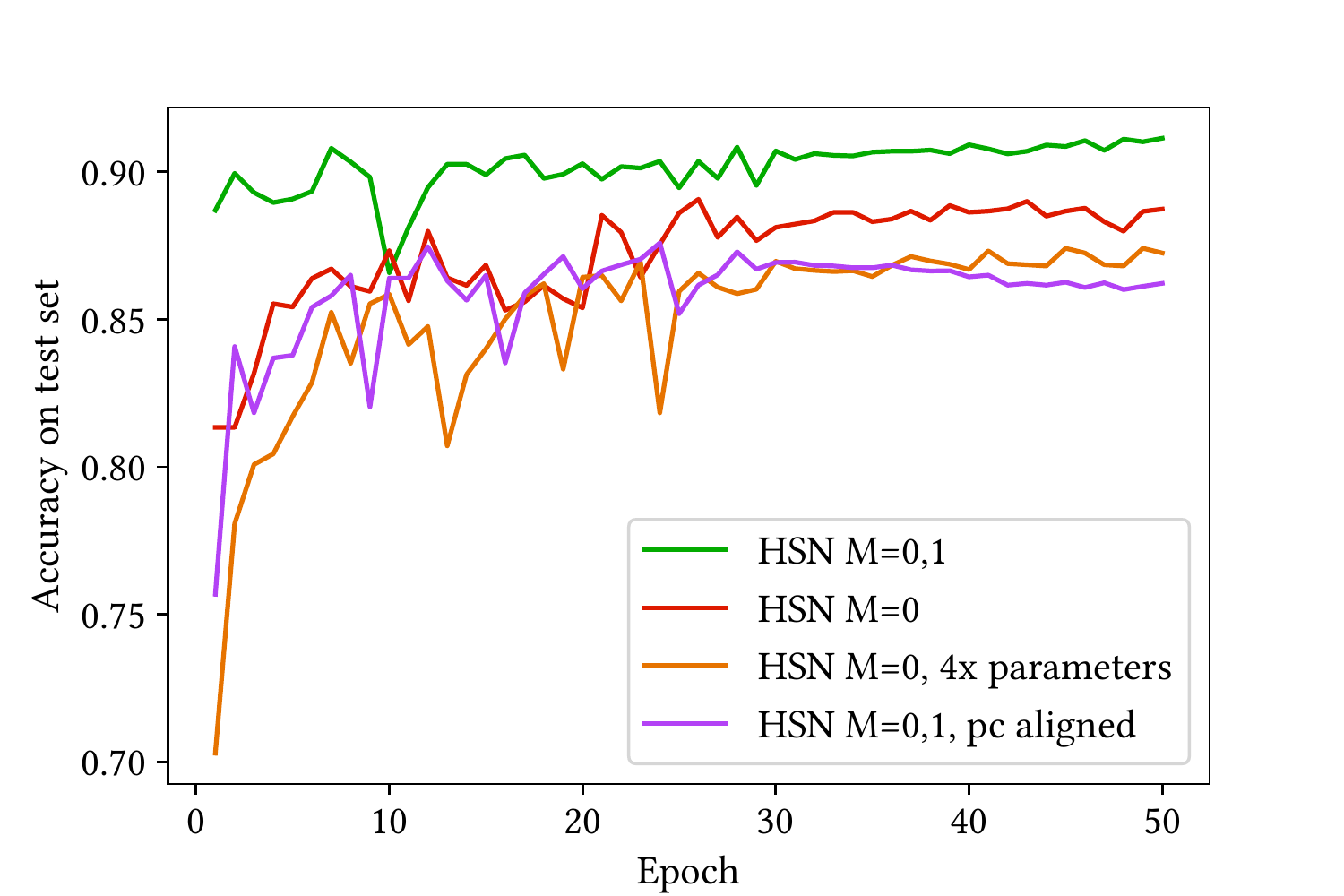}
    \Description{A graph showing four curves. The curve for HSN with two streams starts at 0.88 and remains high, converging to 91.14. The other three curves start lower and are more wobbly.}
    \caption{Validation accuracy per training epoch several configurations of HSN on shape segmentation.}
    \label{fig:shapeseg_m0m1}
\end{figure}

%!TEX root = main.tex

\section{Conclusion}
We introduce Harmonic Surface Networks, an approach for deep learning on surfaces operating on vector-valued, rotation-equivariant features. This is achieved by learning circular harmonic kernels and separating features in streams of different equivariance classes. The advantage of our approach is that the rotational degree of freedom, arising when a filter kernel is transported along a surface, has no effect on the network. The filters can be evaluated in arbitrarily chosen coordinate systems. Due to the rotation-equivariance of the filters, changes in the coordinate system can be recovered after the convolutions have been computed by transforming the results of the convolution. The convolution operator uses this property and always locally aligns the features.
%We introduce Harmonic Surface Networks, an approach for geometric deep learning on surfaces that learns the parameters of a family of filter kernels and applies them on the surface for convolution. The advantage of our approach is that the rotational degree of freedom that arises when a filter kernel is transported along a surface has no effect on the network. The filters can be evaluated in arbitrarily chosen coordinate systems because changes in the coordinate system can be recovered after the convolutions have been computed by transforming the results of the convolution. The convolution operator uses this property and always locally aligns the features. For the design of HSNs, we work with vector-valued features and combine rotation equivariant filters and parallel transport on the surface.

\begin{table}[t]
    \caption{Results of HSN tested on classification for multiple configurations.}
    \label{tab:classification_m0m1}
    \begin{tabular}{lll}
    Method                           & Streams ($M=\ldots$)             & Accuracy \\ \hline \hline
    \multicolumn{1}{l|}{HSN}     & \multicolumn{1}{l|}{0, 1}  & \textbf{\revised{96.1}\%} \\
    \hline
    \multicolumn{1}{l|}{\revised{HSN}}         & \multicolumn{1}{l|}{0}  & \revised{86.1}\%   \\
    \multicolumn{1}{l|}{HSN (pc aligned)}     & \multicolumn{1}{l|}{0, 1}  & \revised{49.7}\%
    \end{tabular}
\end{table}

We implement this concept for triangle meshes and develop convolution filters that have the desired equivariance properties at the discrete level and allow for separating learning of parameters in the radial and angular direction. We demonstrated in our comparisons that HSNs are able to produce competitive or better results with respect to state-of-the-art approaches and analyzed the benefits of rotation-equivariant streams as an addition to rotation-invariant streams and parallel transport for local alignment compared to global alignment.

While we only implemented our method for meshes, every component of HSNs can be transferred to point clouds. We aim to extend our implementation and evaluate the benefits of HSNs for learning on point clouds.

Another promising direction is the application of outputs from the rotation-equivariant stream. We expect that the ability to learn intrinsic vectors on a surface can facilitate new applications in graphics and animation.
%!TEX root = main.tex
\begin{acks}
We would like to thank Nicholas Sharp and Matthias Fey for their work on \textit{Geometry Central} and \textit{PyTorch Geometric}, Adrian Poulenard for sharing results and insights on MDGCNN, Thomas Kipf and Taco Cohen for helpful discussions on geometric deep learning and group-equivariant learning, and the anonymous reviewers for their constructive comments and suggestions. This work has been supported by Stichting Universiteitsfonds Delft and a gift from Google Cloud.
\end{acks}

\bibliographystyle{ACM-Reference-Format}
\bibliography{bibliography}

\end{document}